\documentclass{article}
\usepackage[utf8]{inputenc}
\usepackage[margin=1in]{geometry}

\usepackage{microtype}
\usepackage{graphicx}
\usepackage{subfigure}
\usepackage{booktabs}
\usepackage{hyperref}
\usepackage{amsmath, amsfonts, amssymb, amsthm}
\usepackage{xcolor}
\usepackage{algorithm}
\usepackage{algorithmicx}
\usepackage{algpseudocode}
\usepackage{booktabs}
\usepackage{graphicx}
\usepackage{bbm}
\usepackage{bm}
\usepackage{authblk}
\usepackage{thmtools}
\usepackage[round]{natbib}
\newtheorem{theorem}{Theorem}
\newtheorem{lemma}[theorem]{Lemma}

\newcommand{\blockcomment}[1]{}
\newcommand{\wh}[1]{\widehat{#1}}
\newcommand{\nn}{\nonumber}

\DeclareMathOperator*{\argmin}{argmin}
\newcommand{\nb}{\nabla}
\newcommand{\p}{\partial}
\newcommand{\dth}{\frac{d}{d\th}}
\newcommand{\T}{\top}
\newcommand{\iid}{\stackrel{\tiny{\mathrm{i.i.d.}}}{\sim}}
\newcommand{\deq}{\stackrel{\tiny{\Delta}}{=}}

\newcommand{\smin}{\sigma_{\mathrm{min}}}
\newcommand{\indep}{\perp\!\!\!\perp}

\newcommand{\E}{\mathbb{E}}
\newcommand{\R}{\mathbb{R}}
\renewcommand{\P}{\mathbb{P}}
\newcommand{\I}{\mathbbm{1}}

\newcommand{\D}{\mathcal{D}}
\renewcommand{\L}{\mathcal{L}}
\newcommand{\N}{\mathcal{N}}
\renewcommand{\O}{\mathcal{O}}

\newcommand{\gt}{\tilde{g}}

\newcommand{\hnb}{\hat{\nb}}
\newcommand{\hmu}{\hat{\mu}}

\newcommand{\dmus}{\frac{d\mu^*}{d\th}}
\newcommand{\dmpsi}{\frac{dm}{d\psi}}
\newcommand{\mk}{m^{(k)}}
\newcommand{\mkp}{m^{(k+1)}}
\newcommand{\mkm}{m^{(k-1)}}

\newcommand{\dmk}{\frac{d\mk}{d\th}}
\newcommand{\dmkp}{\frac{d\mkp}{d\th}}
\newcommand{\hdmk}{\widehat{\dmk}}
\newcommand{\hdmkp}{\widehat{\dmkp}}

\newcommand{\Em}{\mathbf{E}_m}
\newcommand{\Emu}{\mathbf{E}_{\mu^*}}
\newcommand{\Eg}{\mathbf{E}_{\nabla}}
\newcommand{\be}{\bm{\e}}
\newcommand{\errmu}{\mathrm{err}_\mu}
\newcommand{\errpsi}{\mathrm{err}_\psi}
\newcommand{\bardmu}{\overline{\Delta \mu}}
\newcommand{\bardpsi}{\overline{\Delta \psi}}
\newcommand{\taylor}{e^{\textrm{Taylor}}}
\newcommand{\Taylor}{E^{\textrm{Taylor}}}

\newcommand{\bpsi}{\bar{\psi}}

\newcommand{\xo}{x^{\textrm{orig}}}
\newcommand{\xc}{x^{\textrm{cur}}}

\renewcommand{\a}{\alpha}
\renewcommand{\b}{\beta}
\renewcommand{\d}{\delta}
\newcommand{\Th}{\Theta}
\renewcommand{\th}{\theta}
\newcommand{\e}{\varepsilon}
\newcommand{\g}{\gamma}
\newcommand{\n}{\eta}
\newcommand{\s}{\sigma}

\newcommand{\opt}{\th_{\mathrm{OPT}}}

\newcommand{\lmax}{\ell_{\mathrm{max}}}

\title{How to Learn when Data Gradually Reacts to Your Model}
\author[1]{Zachary Izzo}
\author[1, 2]{Lexing Ying}
\author[3]{James Zou}
\affil[1]{Department of Mathematics, Stanford University}
\affil[2]{Institute for Computational and Mathematical Engineering, Stanford University}
\affil[3]{Department of Biomedical Data Science, Stanford University}
\date{}

\begin{document}

%

%

\maketitle

\begin{abstract}
A recent line of work has focused on training machine learning (ML) models in the performative setting, i.e. when the data distribution reacts to the deployed model. The goal in this setting is to learn a model which both induces a favorable data distribution and performs well on the induced distribution, thereby minimizing the test loss. Previous work on finding an optimal model assumes that the data distribution immediately adapts to the deployed model. In practice, however, this may not be the case, as the population may take time to adapt to the model. In many applications, the data distribution depends on both the currently deployed ML model and on the ``state'' that the population was in before the model was deployed.
In this work, we propose a new algorithm, Stateful Performative Gradient Descent (Stateful PerfGD), for minimizing the performative loss even in the presence of these effects. We provide theoretical guarantees for the convergence of Stateful PerfGD. Our experiments confirm that Stateful PerfGD substantially outperforms previous state-of-the-art methods.
\end{abstract}

\section{INTRODUCTION}
A recent line of work has sought to study how to effectively train machine learning (ML) models in the presence of performative effects \citep{Perdomo2020}. Performativity describes the scenario in which our deployed model or algorithm effects the distribution of the data or population which we are studying. Such effects can be expected when our model is used to make consequential decisions concerning the population.
As ML becomes ever more ubiquitous across fields, considering these performative effects also grows in importance. 
 
For example, suppose a bank uses a ML model which considers user features---e.g. income, number of open credit lines, etc.---to decide which user should be granted a loan. Based on the original data distribution, the model learns that people with more credit lines open are more likely to repay their loans. After the model is deployed, some users may open more credit lines in order to improve their chances of receiving a loan. In this case, the data distribution has changed as a direct consequence of deploying a specific model. More importantly, the distribution of the outcome---whether or not the person repays his or her loan---given the features has changed, leading to degraded model performance.

Formally, we assume that deploying a model induces a new distribution over test data. The goal of model training under performative distribution shift is to minimize \emph{performative risk}, i.e., the model's loss on the distribution it induces. Recently, \citep{Izzo2021} proposed a ``meta-algorithm'' (performative gradient descent or PerfGD) to accomplish this when the induced data distribution depends only on the deployed model. This amounts to assuming that the data distribution immediately adapts to the deployed model, irrespective of any other conditions. In practice, such a model of performative effects may be overly simplistic. It is likely that the induced distribution will depend not only on the deployed model, but also some notion of the ``state" that the population was in when the model was deployed.
In the loan example, for instance, it will take loan applicants some time to open new credit lines, so we can expect the distribution to change gradually as applicants have more time to adapt, before finally settling to some steady-state distribution for the deployed model.
Optimizing the test loss in the state-dependent performative case has been understudied in the literature, and the addition of a state (which cannot be controlled explicitly, only implicitly) increases the difficulty of the optimization. We propose a novel algorithm and analysis to fill this important gap.

\subsection{Our Contributions}
In this work, we introduce a new algorithm for minimizing the performative risk in the state-dependent performative setting. Our approach is similar in spirit to that of \citep{Izzo2021}, in that it amounts to estimating an appropriate gradient and using it to perform gradient descent. However, unlike \citep{Izzo2021}, we no longer have even direct sample access to the distribution that we care about (the ``long-term" induced distribution, i.e. the distribution which the population will finally settle to over time), and this added technical challenge makes previous algorithms for optimizing the performative risk ineffective.
Indeed, the only way to apply previous approaches directly is to wait for many time steps after each model deployment so that the induced distribution stabilizes to its long-term limit. Our algorithm overcomes this limitation by ``simulating" waiting, without actually needing to do so. We show theoretically that this method accurately captures the behavior of the long-term distribution. Experiments confirm our theory and also show its improvement over existing methods which are not specifically adapted to the stateful setting.

\subsection{Related Work}
It has long been known that changes between training and test distributions can lead to catastrophic failure for many ML models. The general problem of non-identical training and test distributions is known as distribution shift or dataset shift, and there is an extensive literature which seeks to address these issues when training models \citep{DatasetShift2009, Storkey2009-Shift, Moreno2012-Shift}. Much of the work in this area has been devoted to dealing with shifts due to \emph{external} factors outside of the modeler's control, and developing methods to cope with these changes is still a highly active area of research \citep{Koh2021-WILDS}.

\citep{Perdomo2020} proposed studying distribution shifts which arise due to the deployed model itself, referred to as performative distribution shift. They gave two simple algorithms---repeated risk minimization (RRM) and repeated gradient descent (RGD)---which converge to a stable point, i.e. a model which is optimal for the distribution it induces. Other early work in this area also explored stochastic algorithms for finding stable points \citep{Mendler2021-Performative-Stochastic, Drusvyatskiy2020-Performative-Stochastic}.

State-dependence in the performative setting was introduced by \citep{Brown2020}. A notion of optimality in this setting is the minimization of the \emph{long-term} performative loss---that is, finding a model which minimizes the average risk over an infinite time horizon, assuming that we keep deploying that same model. \citep{Brown2020} showed that the RRM procedure introduced by \citep{Perdomo2020} converges to long-term stable points.
RRM and RGD rely on population-level quantities (e.g., minimization of the population-level risk or a population-level gradient). \citep{Li2021-Stateful-Stochastic} extended these results to show that stochastic optimization algorithms also find a performatively stable point in the stateful setting. We remark that these works differ from ours in that they both seek to find a stable point rather than an optimal point (i.e. one which minimizes the test loss), and in general stable points can be far from optimal \citep{Izzo2021, Miller2021}.

\citep{Izzo2021} proposed a method (PerfGD) for computing the performative optimum in the non-state-dependent case. Under parametric assumptions on the performative distribution, they show how to construct an approximate gradient of the performative loss and then use it to perform gradient descent. \citep{Miller2021} also studied optimizing the (stateless) performative loss. The authors quantified when the performative loss is convex and proposed using black-box derivative-free optimization methods to find the performative optimum. For certain classes of performative effects, they also propose a model-based approach to minimizing the performative loss.

A related line of work studies the setting of strategic classification \citep{Hardt2016}, which is a subclass of the general performative setting. In this setting, it is assumed that individual datum react to a deployed model by a best-response mechanism, inducing a population-level distribution shift. \citep{Dong2017} considered optimizing the performative risk in an online version of this problem and for a certain class of best-response dynamics. Other recent work in this area includes developing practically useful tools for modeling strategic behavior, such as differentiable surrogates for strategic responses and regularizers for inducing socially advantageous strategic responses \citep{Levanon2021-Strategic-Practical};
incorporating more realistic limitations on the best-response behavior of the agents \citep{Ghalme2021-Strategic-Dark, Jagadeesan2021-Strategic-Micro}; 
examining the effects of the relative frequency of updates between the modeler and the agents \citep{Zrnic2021-Strategic-Leads};
and studying the statistical and computational complexity of strategic classification in a PAC framework \citep{Sundaram2021-Strategic-PAC}.
While strategic classification offers a wealth of important examples of performative effects, the performative setting is more general as the change in the data need not arise from a best-response mechanism.

The original performative optimization problem can be framed as a derivative-free optimization (DFO) \citep{Flaxman2005} problem with a noisy function value oracle. In the stateful case, however, we no longer even have an unbiased noisy oracle for the function we wish to optimize (the long-term performative risk), making black-box DFO algorithms ineffective.

\section{PROBLEM SETUP} \label{sec: setup}
We refer readers unfamiliar with the performative literature to the introductory sections of \cite{Perdomo2020} and \cite{Izzo2021} for a complete discussion of the original (stateless) performative prediction setting.

We consider a generalization of the performative prediction problem \citep{Perdomo2020}, introduced by \citep{Brown2020} and referred to as ``stateful" performativity.
Let $\Th$ denote the set of admissible model parameters and $\mathcal{Z}$ denote the data sample space. We assume that there is a \emph{distribution map} $\D: \Th \times \mathcal{M}(\mathcal{Z}) \rightarrow \mathcal{M}(\mathcal{Z})$, where $\mathcal{M}(\mathcal{Z})$ denotes the set of probability measures on $\mathcal{Z}$.
If $\rho_t$ denotes the data distribution at time $t$ and $\th_t$ denotes the model that we deploy at time $t$, then we have
$$\rho_t = \D(\th_t, \rho_{t-1}).$$
That is, the data distribution at time $t$ is a function of the model we deployed, as well as the previous state that the population was in (encoded by the previous distribution $\rho_{t-1}$).
Note that this setting is strictly more general than the original setup of \citep{Perdomo2020}, in which $\rho_t = \D(\th_t)$ depended only on the deployed model. This generalization captures the fact that in practice, it is unlikely that the population we are modeling will immediately snap to a new distribution upon deployment of a new model. In general, it will take the distribution some time to adapt.

Under reasonable regularity conditions on $\D$, if we define $\th_t \equiv \th$ for all $t$,
then there exists a limiting distribution $\rho^*(\th) = \lim_{t\rightarrow\infty} \rho_t$. (See Claim 1 of \citep{Brown2020} for sufficient conditions.) That is, $\rho^*(\th)$ describes the limiting distribution if we continue to deploy $\th$ for all time steps $t$, and it is assumed that this distribution is independent of the initial state. If we define the long-term performative loss $\L^*(\th) = \E_{\rho^*(\th)}[\ell(z; \th)]$, then a sensible goal is to compute the long-term optimum $$\opt \deq \argmin_{\th \in \Theta} \L^*(\th).$$ This is similar to the problem addressed in \citep{Izzo2021}, except now we do not even have direct sample access to $\rho^*(\th)$.

Throughout the paper, we will assume that $\rho_t$ belongs to a parametric family with (unknown) parameter $\mu_t$ and corresponding density $p(\cdot, \mu_t)$. For concreteness, one may think of the distribution as a mixture of Gaussian with fixed covariances $\Sigma$ but unknown means $\mu$, but we remark that our techniques should be viewed more as a ``meta-algorithm" whose details can be directly applied to other parametric distributions. In this setting, rather than the distribution map $\D$, we can equivalently consider the parameter map $m$, where $\mu_t = m(\th_t, \mu_{t-1})$, and then $\rho_t$ corresponds to the parametric distribution with parameter $\mu_t$ (i.e., $\rho_t$ has density $p(\cdot, \mu_t)$). Analogously to the long-term distribution assumption, we will assume that for every fixed $\th$ and any starting $\mu$, there is a long-term parameter $\mu^*(\th) = \lim_{k\rightarrow\infty} m^{(k)}(\th, \mu)$, where $m^{(0)}(\th, \mu) = \mu$ and $m^{(k)}(\th, \mu) = m(\th, m^{(k-1)}(\th, \mu))$ for $k \geq 1$. That is, $m^{(k)}(\th, \mu)$ denotes the distribution parameters after model $\th$ has been deployed for $k$ steps, starting from the distribution with parameters $\mu$. For simplicity, we will assume that the model parameters $\th$ as well as the distribution parameters $\mu$ are both $d$-dimensional vectors: $\th, \, \mu \in \R^d$, but we emphasize that this is for notational convenience and is not required.

Algorithm \ref{alg: deploy} describes the interaction model for our problem in terms of the parameterized distribution. Here we have assumed that, given a sample $Z = \{z_i\}_{i=1}^n$ from the distribution with parameter $\mu$, there is some method (e.g. maximum likelihood) for estimating  $\hat{\mu}(Z)$. Since there is a large literature on parametric inference, we consider $\hat{\mu}$ as provided.

\begin{algorithm}[h!]
\caption{Deployment and sampling model} \label{alg: deploy}
\begin{algorithmic}
\Procedure{Deploy\&Sample}{$\th_t$, $\mu_{t-1}$}
\State Deploy $\th_t$
\State Population reacts: $\mu_t \gets m(\th_t, \mu_{t-1})$
\State Collect samples: $Z_t \gets \{z^{(t)}_i\}_{i=1}^{n_t}$, $z^{(t)}_i \iid \P_{\mu_t}$
\State Estimate $\mu_t$: $\hat{\mu}_t \gets \hat{\mu}(Z_t)$
\State\Return $\hat{\mu}_t$
\EndProcedure
\end{algorithmic}
\end{algorithm}

Lastly, we will use $\p_i f$ to denote the derivative of a function $f$ with respect to its $i$-th argument. So for instance, $\p_1 m(\th, m(\th, \mu_t))$ means the derivative of $m(\th, m(\th, \mu_t))$ only with respect to the $\th$ appearing in the first argument (before the comma) even though $\th$ appears in $m$ in the second argument as well.

\subsection{Why is the State-Dependent Case More Challenging?}
As mentioned above, adding state-dependence to the performative dynamics presents a much more realistic model of performative effects likely to arise in reality. However, this added realism makes the optimization problem significantly more challenging. Indeed, we no longer even have direct access to an unbiased estimate for the function that we wish to minimize---the long-term performative loss---, as we cannot observe the long-term performative loss simply by deploying our model once.
The increase in problem complexity is akin the gap between bandit problems and reinforcement learning/Markov decision processes. Thus although our setting may seem similar to that of \citep{Izzo2021} at face value, the state-dependent case is in fact a highly nontrivial advancement both in terms of the practical validity of the model and the technical/theoretical difficulty of solving the problem. Therefore, the problem demands novel algorithms and analysis, which we introduce here.

\section{STATEFUL PERFGD}
Our approach is to estimate the (long-term) performative gradient and then use this estimate to do approximate gradient descent. In an ideal world, we wish to compute the gradient
\begin{align*}
    \nb_\th \L^*(\th) &= \nb_\th \left[\int \ell(z; \th) p(z; \mu^*(\th)) \, dz \right] \\
    &= \int \nb_\th \ell(z; \th) p(z; \mu^*(\th)) \, dz + \int \ell(z; \th) \dmus^\T \nb_\mu p(z; \mu^*(\th)) \, dz.
\end{align*}
(Note: $\nb_\mu p$ denotes the gradient of the density $p$ with respect to its $\mu$ argument. In terms of the $\p_i$ notation, we have $\nb_\mu p = \p_2 p^\T$.)
There are two unknown quantities in this expression: $\mu^*(\th)$ and $\dmus$. The various subroutines in the algorithm are all aimed at estimating these unknown quantities. As we no longer have direct sample access to the long-term distribution $\mu^*(\th)$, the steps needed to estimate the long-term performative gradient are different from \citep{Izzo2021}, and the error analysis is more involved. There are two main components in the algorithm. First, we use the most recent $H$ steps in the training trajectory to estimate the derivatives of the update function $m$ (Algorithm \ref{alg: finite difference}), which can then be used to estimate $\dmus$ (Equation \eqref{eq: lt jac}). With this estimate in hand, we can compute an estimate of the total gradient of the long-term loss and take a gradient descent step (Equation \eqref{eq: lt grad} and Algorithm \ref{alg: spgd}). The precise steps for the algorithm are given below. In the following, $\psi_t = [\th_t^\T, \, \mu_{t-1}^\T]^\T$ denotes the full input to $m$ at time $t$, and for any collection of vectors $v_i$, $v_{i:j}$ denotes the matrix with columns $v_i, \, v_{i+1}, \, \ldots, \, v_j$.

\begin{algorithm}[h!]
\caption{Estimating $\p_i m$} \label{alg: finite difference}
\begin{algorithmic}
\Require Estimation horizon $H$
\Procedure{EstimatePartials}{$\psi_{t-H:t}$, $\hat{\mu}_{t-H:t}$}
\State $\Delta \psi \gets \psi_{t-H:t-1} - \psi_t \mathbf{1}_H^\T$
\State $\Delta \mu \gets \hat{\mu}_{t-H:t-1} - \hat{\mu}_t \mathbf{1}_H^\T$
\State $[\wh{\p_1 m}, \: \wh{\p_2 m}] \gets (\Delta \mu)(\Delta \psi)^\dagger$
\State\Return $\wh{\p_1 m}$, $\wh{\p_2 m}$
\EndProcedure
\end{algorithmic}
\end{algorithm}

\begin{algorithm}[h!]
\caption{Stateful PerfGD} \label{alg: spgd}
\begin{algorithmic}
\Require Estimation horizon $H$, perturbation size $\s^2$, learning rate $\n$

\State Initialize for $d$ steps and record $\hat{\mu}_{t-1}$, $\th_t$, and $\hat{\mu}_t$ 

\While{not converged}
\State $\hat{\mu}_t \gets $\textsc{Deploy\&Sample}($\th_t$, $\mu_{t-1}$)
\State $\wh{\p_1 m}, \wh{\p_2 m} \gets $\textsc{EstimatePartials}$(\psi_{t-H:t}, \: \hat{\mu}_{t-H:t})$
\State $\wh{\dmus} \gets \textsc{EstimateLTJac}(\wh{\p_1 m}, \: \wh{\p_2 m})$
\State $\wh{\nb\L^*_t} \gets \textsc{EstimateLTGrad}(\th_t, \: \hat{\mu}_t, \: \wh{\dmus})$
\State $\th_{t+1} \gets \th_t - \n (\wh{\nb\L^*_t} + g_t)$, $g_t \sim \N(0, \s^2 I)$
\State $t \gets t+1$
\EndWhile
\end{algorithmic}
\end{algorithm}
The other estimation functions \textsc{EstimateLTJac} and \textsc{EstimateLTGrad} are given by
\blockcomment{
\begin{align}
    \textsc{EstimateLTJac}&(\wh{\p_1 m}, \wh{\p_2 m}, k) = \label{eq: lt jac}\\
    &(I - \wh{\p_2 m})^{-1} (I - (\wh{\p_2 m})^k) \wh{\p_1 m}, \nonumber
\end{align}
}
\begin{align}
    \textsc{EstimateLTJac}(\wh{\p_1 m}, \wh{\p_2 m}) &= (I - \wh{\p_2 m})^{-1} \wh{\p_1 m} \label{eq: lt grad} \\
    \textsc{EstimateLTGrad}(\th_t, \: \hat{\mu}_t, \: \wh{\dmus}) &= \int \nb_\th \ell(z; \th_t) p(z; \hat{\mu}_t) \, dz + \int \ell(z; \th_t) \wh{\dmus}^\T \nb_\mu p(z; \hat{\mu}_t) \, dz \label{eq: lt jac}
\end{align}
Next, we give the basic motivation for each step of this algorithm. Algorithm \ref{alg: finite difference} estimates the derivatives of $m$ using finite difference approximations gathered from the optimization trajectory so far. The columns $\Delta \psi$ are the differences in the input of $m$, and the columns of $\Delta \mu$ are the corresponding differences in the output of $m$. We then estimate the derivatives of $m$ by solving the matrix equation $(\Delta\textrm{output}) \approx (\textrm{Jacobian}) \cdot (\Delta \textrm{input})$. The estimation horizon $H$ is a hyperparameter which should be tuned via standard techniques. In our proofs, we require that $H$ be polynomially larger than the dimension $d$, but in practice we find that choosing $H \in [2d, 3d]$ or just using the entire previous trajectory for this step works well. We also note that $H\geq 2d$ should be enforced so that the ``input difference matrix'' $\Delta \psi \in \R^{2d \times H}$ will have a right inverse.

The formula for \textsc{EstimateLTJac} arises from the recursive definition of $m^{(k)}$ (see Section~\ref{sec: setup}). Taking a derivative with respect to $\th$, unrolling the recursion, and sending $k\rightarrow\infty$ leads to the formula \eqref{eq: lt jac}.

The formula for \textsc{EstimateLTGrad} is derived by taking a derivative of the long-term performative loss, recalling that this is an expectation with respect to the \emph{known} density $p$ with unknown parameter $\mu^*$. As we do not know $\mu^*$ or its Jacobian $\dmus$, we simply subsitute our best approximations for each of these ($\hat{\mu}_t$ and $\widehat{\dmus}$, respectively) to obtain the formula \eqref{eq: lt grad}. In Section \ref{sec: theory}, we bound the error of our approximation and show that it vanishes as the number of steps increases and as the error in $\hat{\mu}$ goes to 0.
%
This yields an estimate for the long-term performative gradient, which we then use to take an approximate gradient descent step. The Gaussian perturbations $g_t$ are a technical necessity which borrows ideas from smoothed analysis \citep{Sankar2006-Smoothed}. They ensure that the optimization trajectory has traveled enough in each direction so that the derivatives of $m$ can be estimated even in the presence of errors in $\hat{\mu}$ and can often safely be omitted in practice.
See Appendix~\ref{appendix: derivation} for a full derivation of the algorithm.

\subsection{Performativity through Low-Dimensional Statistics} \label{sec: bottleneck}
Performativity in ML is primarily concerned with changes in human populations as the result of a deployed model. Unless the population being model consists mostly of data scientists, it is unlikely that the constituent individuals will have a reaction based on the particular parameters of the model. Instead, individuals (and therefore the distribution of the population on the whole) likely modify their behavior based on a low-dimensional proxy, such as a credit score or classification probability. If it is the case that the distribution shift depends on a low-dimensional statistic, then we can still apply stateful PerfGD for a very high-dimensional model (e.g. a neural network) without incurring a large error due to the high dimension.

We formalize this intuition as follows. Suppose that the stateful parameter map actually takes the form
$ \mu_t = m(\th_t, \mu_{t-1}) = \bar{m}(s(\th_t, \mu_{t-1}), \mu_{t-1}), $
where $s$ is a \emph{known} score function with $s(\th, \mu) \in \R^{d_s}$ and $d_s \ll \dim(\th)$. In this case, we may estimate the partials of $\bar{m}$ with respect to $s$ and then use the chain rule to compute the partial of $m$ with respect to $\th$, yielding
$\p_1 m(\th_t, \mu_{t-1}) = \p_1 \bar{m}(s_t, \mu_{t-1}) \, \p_1 s(\th_t, \mu_{t-1}),$
where $s_t = s(\th_t, \mu_{t-1})$. Note that since $s$ is known as a function of $\th$ and $\mu$, computing $\p_1 s(\th_t, \mu_{t-1})$ just requires estimating $\mu_{t-1}$ (which we have assumed is easy) and we instead need only estimate the derivative $\p_1 \bar{m}$.
This is a derivative with respect to $d_s$ variables, whereas in general for this step we must compute a derivative with respect to $\dim(\th)$ variables. When $d_s \ll \dim(\th)$, this can make the derivative estimation task signficantly easier. We can then plug this estimator for $\p_1 m$ into Algorithm \ref{alg: spgd} and proceed as usual.

\section{THEORETICAL GUARANTEES} \label{sec: theory}
In this section, we quantify the performance of Stateful PerfGD theoretically. We require the following:
\begin{enumerate}
    \item $\| \p_1 m(\th, \mu) \| \leq B$ \label{assumption: m lip}
    \item $\| \p_2 m(\th, \mu) \| \leq \d < 1$ \label{assumption: contraction}
    \item $\smin(\p_2 m(\th, \mu)) \geq \alpha$ \label{assumption: d2 lb}
    \item $|\ell(z; \th)|, \, \|\nb_\th \ell(z; \th)\| \leq \lmax$ \label{assumption: loss}
    \item $\| \nb^2 m \| \leq C$, where $\nb^2 m$ is the tensor of second derivatives of $m$, and $\| \nb^2 \L^*(\th) \| \leq L$. \label{assumption: 2nd derivs}
    \item $\textrm{diam}(\{\mu\}) \leq D$, where $\{\mu\}$ denotes the set of all possibile stateful distribution parameters. \label{assumption: diam}
    \item The estimator $\hat{\mu}$ satisfies $\| \hat{\mu}_t - \mu_t \| \leq \bm{\e}$. \label{assumption: mu err}
\end{enumerate}
For simplicity we will also assume that the stateful performative distribution is a Gaussian with unknown mean, i.e. the distribution parameters $\mu$ are just the mean of the Gaussian and $p(z, \mu)$ denotes the Gaussian density of $z$. We will also assume that the covariance is fixed and nondegenerate. The Gaussian assumption simplifies some of the already extensive calculations, but we remark that all of the results still hold for any continuous distribution with sufficiently light (e.g. sub-Gaussian) tails and smooth dependence on the distribution parameter.

With the exception of Assumptions \ref{assumption: contraction} and \ref{assumption: d2 lb}, the above are all standard smoothness assumptions \citep{Perdomo2020, Brown2020, Izzo2021, Miller2021}.
Assumption \ref{assumption: contraction} is a sufficient condition to guarantee that $m^{(k)}(\th, \mu) \rightarrow \mu^*(\th)$ independent of $\mu$, and, when combined with Assumption \ref{assumption: diam}, gives us a bound on the speed of this convergence. On the other hand, Assumption \ref{assumption: d2 lb} ensures that we are able to perturb $\mu_t$ by perturbing $\th_t$; without this, estimating $\partial_2 m$ will be impossible. Finally, we remark that Assumption \ref{assumption: mu err} can easily be converted into a high-probability statement depending on the size of the sample collected at each step. For instance, in the case of a Gaussian mean, we have $\bm{\e} = \O(\sqrt{\log(T) / n})$ by a simple Gaussian concentration/union bound argument.

In all of the following statements, $\O$ hides depedence on any of the problem-specific constants introduced in the assumptions, as well as dependence on the problem dimension and the failure probability $\g$. Our concern is how the error of the method behaves as the time horizon $T \rightarrow \infty$ and the estimation error $\bm{\e} \rightarrow 0$. $\widetilde{\O}$ hides these same constants as well as log factors in $T$. Our main theoretical result is the following convergence theorem for Stateful PerfGD:
\begin{restatable}{theorem}{main} \label{thm: main}
Let $T$ be the number of deployments of Stateful PerfGD, and for each $t$ let $\nb \L^*_t = \nb \L^*(\th_t)$. Then for any $\g > 0$, there exist intervals $[\n_{\min}, \n_{\max}]$ and $[\s_{\min}, \s_{\max}]$ (which depend on $T$ and the estimation error $\be$) such that for any learning rate $\n$ in the former and perturbation size $\s$ in the latter interval, with probability at least $1-\g$, the iterates of Stateful PerfGD satisfy
\begin{equation*}
    \min_{1\leq t \leq T} \| \nb \L^*_t \|^2 = \widetilde{\O}(T^{-1/5} + \bm{\e}^{1/5}).
\end{equation*}
%
\end{restatable}
Theorem \ref{thm: main} shows that Stateful PerfGD finds an approximate critical point. As Stateful PerfGD can be viewed as instantiating gradient descent on the long-term performative loss, and gradient descent is known to converge to minimizers \citep{Lee2016-GD}, Stateful PerfGD will converge to an approximate local minimum. In the case that the long-term performative is convex, Stateful PerfGD will converge to the optimal point.

While the full proof of Theorem \ref{thm: main} requires extensive calculation, the structure of the proof is intuitive and we outline it below. We begin by bounding the error of the finite difference approximation in Algorithm \ref{alg: finite difference}.
\begin{restatable}{lemma}{fderr} \label{thm: finite diff error}
Suppose that $\|\hnb \L^*_s\|$ are bounded by a constant for $s < t$. Let $\frac{dm}{d\psi}$ denote the true Jacobian of $m$ with respect to its input, and let $\widehat{\frac{dm}{d\psi}} = [\widehat{\p_1 m}, \: \widehat{\p_2 m}]$ denote the estimator from Algorithm \ref{alg: finite difference}. Then for appropriate choices of $\n$, $\s$, and $H$, we have
$$ \| \widehat{\frac{dm}{d\psi}} - \frac{dm}{d\psi} \| \leq \tilde{\O}(\frac\n\s + \frac{\be}{\n\s^2}) \equiv \mathbf{E}_m.$$
\end{restatable}
Here, a smaller step size results in smaller error from the finite difference approximation, but magnifies any error in $\hat{\mu}$.
Next, we analyze how the error on our estimate of the short-term derivatives translates to error on our estimates of the long-term derivatives.
\begin{restatable}{lemma}{ltjacerr} \label{thm: lt jac error}
The long-term Jacobian estimate $\wh{\dmus}$ from Eq. \ref{eq: lt jac} satisfies
$$\|\wh{\dmus} - \dmus\| = \widetilde{\O}(\n + \mathbf{E}_m) \equiv \mathbf{E}_{\mu^*},$$
where $\mathbf{E}_m$ is the upper bound on the error in estimating the Jacobian of $m$ from Lemma \ref{thm: finite diff error}.
\end{restatable}
Note that the error on the long-term Jacobian estimate also depends on the distances $\| \mu_t - \mu^*(\th_t) \|$. A smaller learning rate gives the distribution time to adapt \emph{during} training but without needing to wait, making these distances shrink. This can be thought of as similar to multiscale considerations in the study of PDEs \citep{E2011-Multiscale}.
Next, we show that the estimation errors on $\widehat{\dmus}$ and $\hat{\mu}_t$ remain small when we use them to estimate $\nb \L^*_t$.
\begin{restatable}{lemma}{ltgraderr} \label{thm: lt grad error}
The estimator $\wh{\nb \L^*_t}$ from Eq. \eqref{eq: lt grad} satisfies
$ \| \wh{\nb \L^*_t} - \nb \L^*_t \| = \widetilde{\O}(\n + \mathbf{E}_{\mu^*}), $
where $\mathbf{E}_{\mu^*}$ is the error bound on $\wh{\dmus}$ from Lemma \ref{thm: lt jac error}.
\end{restatable}
We show that the errors $\mathbf{E}_m$ and $\mathbf{E}_{\mu^*}$ from Lemmas~\ref{thm: finite diff error} and \ref{thm: lt jac error} vanish at a polynomial rate as $T \rightarrow \infty$ and $\bm{\e} \rightarrow 0$, so that the error in our long-term loss gradient also vanishes as the step size decreases.
Finally, we use a standard analysis of gradient descent on $L$-smooth functions which allows for error in the gradient oracle.
\begin{restatable}{lemma}{gderr} \label{thm: gd with errors}
Let $h$ be any $L$-smooth function and let $\wh{\nb h}$ be a gradient oracle with bounded error: $\| \wh{\nb h}(x) - \nb h(x) \| \leq \mathbf{e}$, and assume that $\mathbf{e} = o(1)$. Then for $\n$ sufficiently small, the iterates $x_t$ of gradient descent with gradient oracle $\wh{\nb h}$ satisfy
$ \min_{1 \leq t \leq T} \| \nb h(x_t) \|^2 = \O(\frac{1}{T\n} + \mathbf{e}). $
\end{restatable}
Combining Lemmas~\ref{thm: lt jac error}-\ref{thm: gd with errors} yields a set of dependencies on $\n$, $\s$, $\bm{\e}$, and $T$ which can be balanced to prove Theorem~\ref{thm: main}. All of the proofs can be found in Appendix~\ref{appendix: proofs}.

\section{EXPERIMENTS} \label{sec: experiments}
In this section, we conduct experiments for all of the relevant methods, showing Stateful PerfGD's improvements over existing algorithms. First, we discuss the algorithms against which we will compare.
\subsection{Previous Algorithms} \label{sec: previous algs}
\paragraph{Repeated Gradient Descent (RGD)}
This method was introduced in \citep{Perdomo2020} and refers to simply taking a gradient of the loss assuming that the distribution is fixed, then updating the model with a gradient descent step and redeploying. \citep{Li2021-Stateful-Stochastic} showed that RGD converges to a stable point in the long run in the stateful performative setting. Since the stateless performative problem is a subclass of the stateful one, there are cases where a stable point can be arbitrarily far from an optimal point. (See \S2.2 of \citep{Izzo2021}.) 

\paragraph{PerfGD (PGD)}
If we repeatedly deploy each model $\th$ until the induced distribution settles to its long-term state, then we can directly apply PerfGD from \citep{Izzo2021}. While this method will eventually converge if we wait long enough at each step, we will have to deploy many suboptimal models if the induced distribution takes a long time to settle, leading to losses for the user.

\paragraph{Black-Box Derivative-Free Optimization (DFO)}
Black-box DFO seeks to optimize a function given only a function value oracle and no direct access to gradients or higher-order derivatives of the function to be optimized \citep{Flaxman2005}. The non-stateful performative prediction setting is a special case of this general problem, and black-box DFO algorithms can obtain reasonable results for non-stateful performative prediction \citep{Miller2021}. In the stateful setting, however, we no longer have a function value oracle for the long-term performative loss, so we expect black-box DFO methods to have degraded performance (if they work at all). We could take the same approach as mentioned above with PerfGD, i.e. deploying each model many times until the distribution settles to its long-term state. We note that since this method deploys perturbed versions of its best internal estimate, the cost in terms of suboptimal model deployments can be even greater than that incurred by adapting PerfGD to the stateful setting.

In all of the following figures, the solid lines denote the mean of the reported statistic and the shaded error regions denote the standard error of the mean. OPT denotes the long-term optimal loss, STAB denotes the loss of the performatively stable point, and SPGD denotes Stateful PerfGD.
For details on the specific constants and hyperparameters, refer to Appendix~\ref{appendix: experiment details}.

\subsection{Linear map $m$} \label{sec: linear m experiment}
We begin with a simple case with a linear point loss $\ell(z, \th) = -z^\T \th$. The long-term performative loss is $\L^*(\th) = -\mu^*(\th)^\T \th$.
We take the mean update function to be $m(\th, \mu) = \d \mu^*(\th) + (1 - \d)\mu$ and set $\mu^*(\th) = A\th + b$ for some fixed $\d \in (0, 1)$, a fixed matrix $A$ and a fixed vector $b$. When $A \prec 0$, the long-term optimal point can be computed exactly as $\opt = -\frac12 A^{-1}b$.

Figure \ref{fig: performance vs. statefulness} compares the performance of SPGD with the other algorithms as the ``amount of statefulness" varies. The $x$-axis is the number of deployments required before $99\%$ of the effect of the previous mean has been removed, which corresponds to a particular $\d$. (If $\th$ is deployed for $k$ steps starting from distribution mean $\mu$, then the mean is $(1 - (1-\d)^k) \mu^*(\th) + (1-\d)^k \mu$, so we want $(1-\d)^k = 0.01$ or $\d = 0.01^{1/k}$.) The $y$-axis shows the best (over a grid search of hyperparameters for each method) final performance as a fraction of OPT achieved by each of the methods after 50 model deployments. Note that since OPT is negative, a lower final loss corresponds to a larger fraction of OPT.

While PGD and DFO make some progress towards OPT, their performance suffers even in the presence of mild state-dependence and continues to degrade as the ``statefulness" of the dynamics increases. These methods must choose between longer wait times or larger errors in estimating the long-term distribution. By specifically accounting for the state dependence of the problem, SPGD maintains near-optimal performance even as the distribution takes longer to settle. For comparison, in the setting with $k=64$ (the right-most point in Figure~\ref{fig: performance vs. statefulness}), the distribution takes 64 steps to settle, but we have only allowed 50 deployments for optimization. By \emph{simulating} rather than \emph{waiting} for the distribution to adapt, SPGD still reaches a near-optimal point quickly.

\begin{figure}
    \centering
    \includegraphics[width=.5\linewidth]{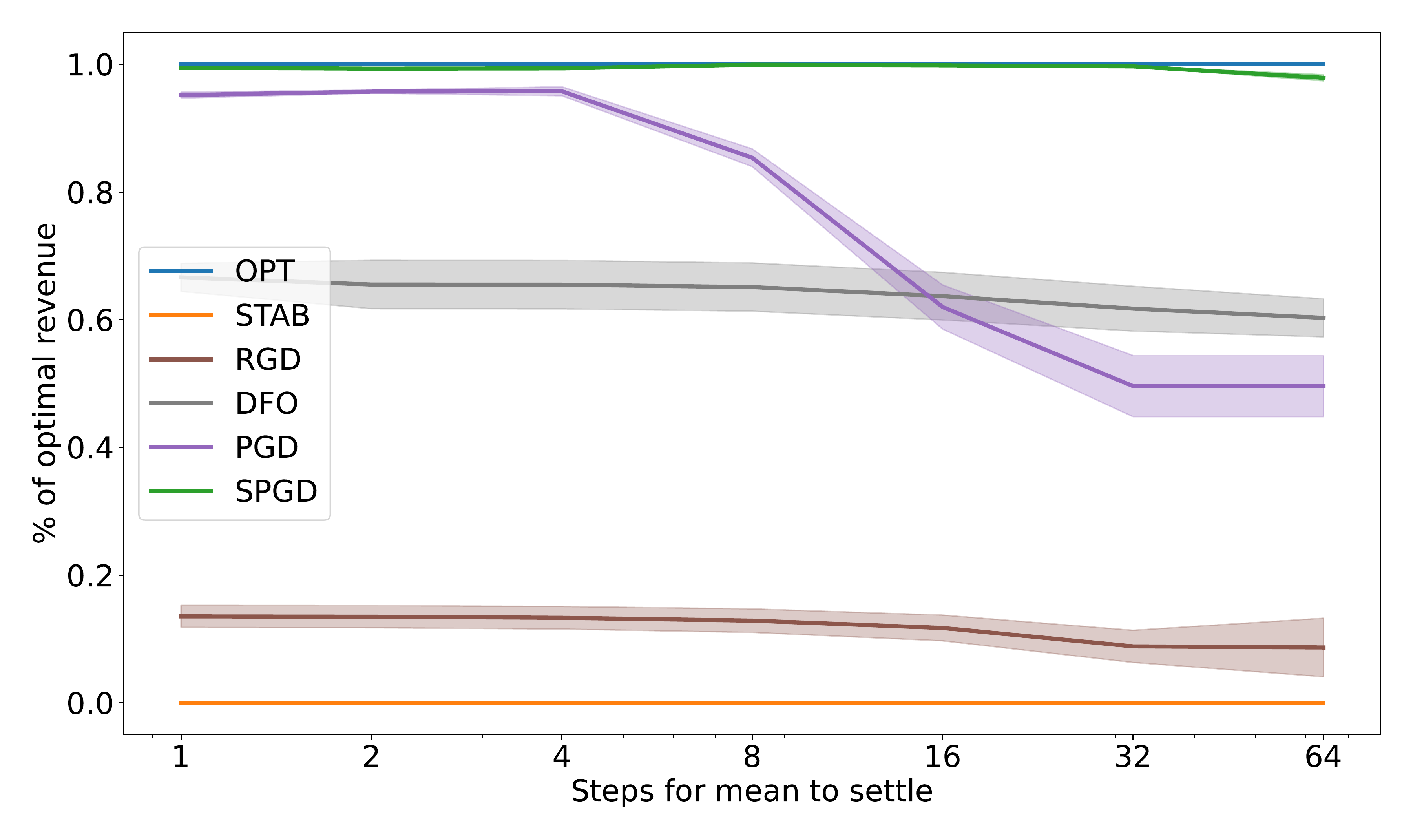}
    \caption{Fraction of optimal performance obtained by each method (higher is better, blue line is the best possible). SPGD is able to reach OPT even when the short-term mean is highly state-dependent. The other methods fail to find OPT, and their performance degrades as the statefulness of the problem increases.}
    \label{fig: performance vs. statefulness}
\end{figure}

\subsection{Nonlinear map $m$} \label{sec: nonlinear m experiment}
We alter the first example so that the rate of convergence to the long-term mean depends on the current mean and varies by coordinate. In particular, we take
$ m(\th, \mu)[i] = \d^{\mu[i]^2} \mu^*(\th)[i] + (1 - \d^{\mu[i]^2}) \mu[i], $
with $\mu^*(\th) = A\th + b$ as before. Here $v[i]$ denotes the $i$-th component of a vector $v$. The long-term performative loss and optimal point are the same as before since $m^{(k)}(\th, \mu) \rightarrow A\th + b$, but $\p_i m$ are more challenging to estimate.

Refer to Figure \ref{fig: nonlinear}. Here we plot the results for a fixed $\d$ so that we can see the training dynamics within a given scenario. The $x$-axis is the training iteration and the $y$-axis is the test loss at that iteration.
In spite of the increased complexity in the derivatives of $m$, we see that SPGD manages to find $\opt$, while the other methods have poorer performance.

\begin{figure}[h!]
\centering
\includegraphics[width=.5\linewidth]{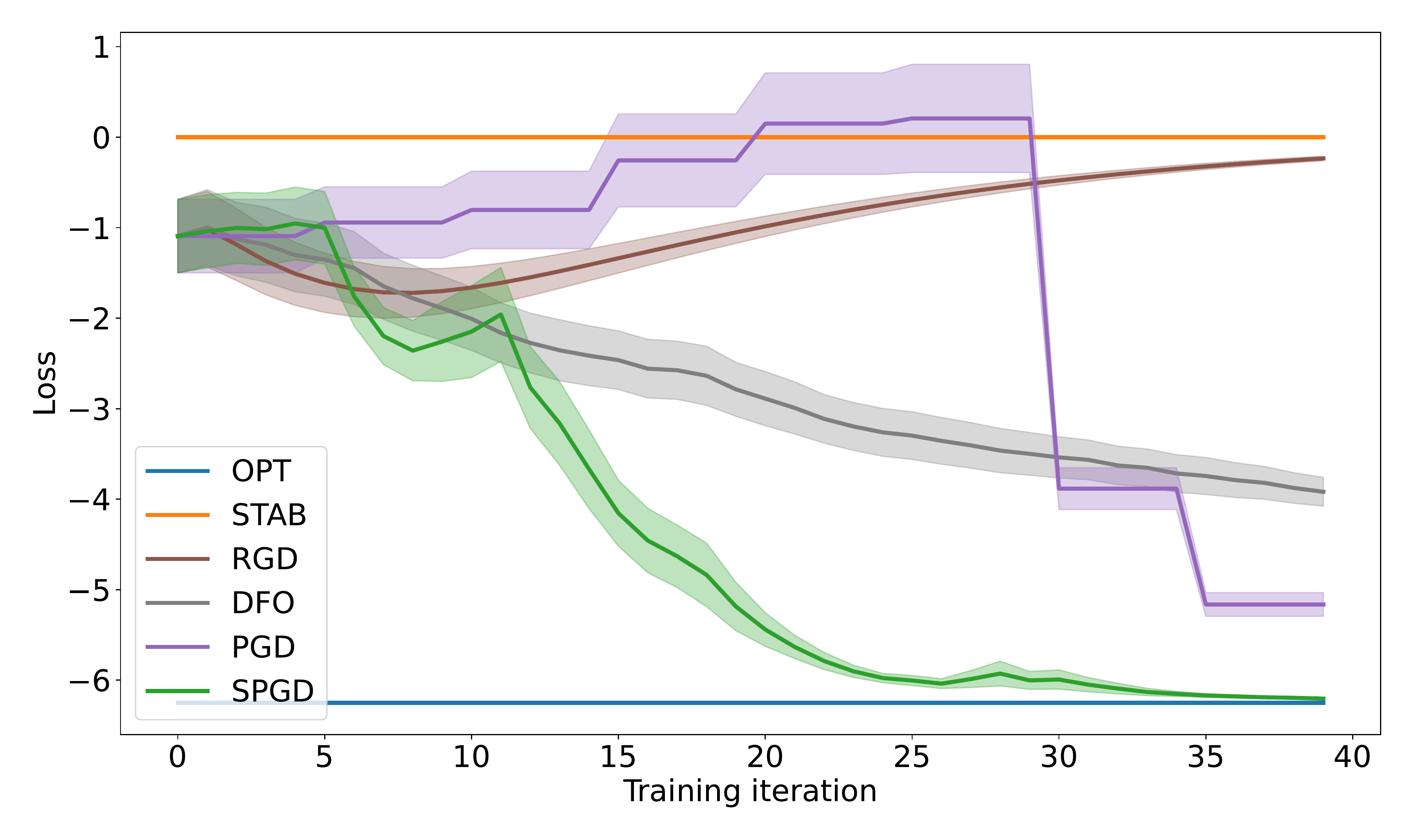}
\caption{Results for nonlinear map $m$. SPGD is able to converge to $\opt$, while the other methods fail to cope with the state-dependence.} \label{fig: nonlinear}
\end{figure}

\subsection{Classification} \label{sec: spam}
We next consider a more realistic spam classification simulation which was studied in \citep{Izzo2021}. The dynamics for this experiment arise when the spammers behave strategically according to the following (state-dependent) cost function. Each spammer has some original message, denoted by the features $\xo$, that they would like to send. This should not be thought of as an actual saved message, but rather encoding the information (e.g. a virus, scam, etc.) that they want to deliver to their victims. Their message also has a current form, denoted by the features $\xc$. We follow the strategic classification framework \citep{Hardt2016}, where each spammer updates their message by maximizing their utility minus a modification cost, given by
$$ \max_x \underbrace{-x^\T\th}_\textrm{Utility} - \underbrace{\frac\a2\|x - \xo\|^2}_{\textrm{Long-term cost}} - \underbrace{\frac\b2\|x - \xc\|^2}_{\textrm{Short-term cost}}. $$
The utility corresponds to the spammers' desire to receive a negative (non-spam) classification from our deployed logistic model. If we take $\a = \e^{-1}$ and $\b = \e^{-1}(\d^{-1} - 1)$, we get the individual dynamics $\xc \mapsto \d(\xo - \e\th) + (1-\d)\xc$, which in turn yields the mean map $m(\th, \mu) = \d(\mu_\textrm{orig} - \e\th) + (1-\d)\mu.$ The point loss for this experiment is ridge-regularized cross-entropy.

The results are shown in Figure \ref{fig: classification}. DFO, PGD, and SPGD are all able to eventually find $\opt$, but by simulating the long-term change in the distribution, SPGD is able to find this optimum in only 6 deployments. PGD requires long waits for the mean to settle in order to converge (leading to the flat regions in the training curve), and DFO requires deploying highly perturbed models to overcome the noise in the mean estimation.

\begin{figure}
    \centering
    \includegraphics[width = .5\linewidth]{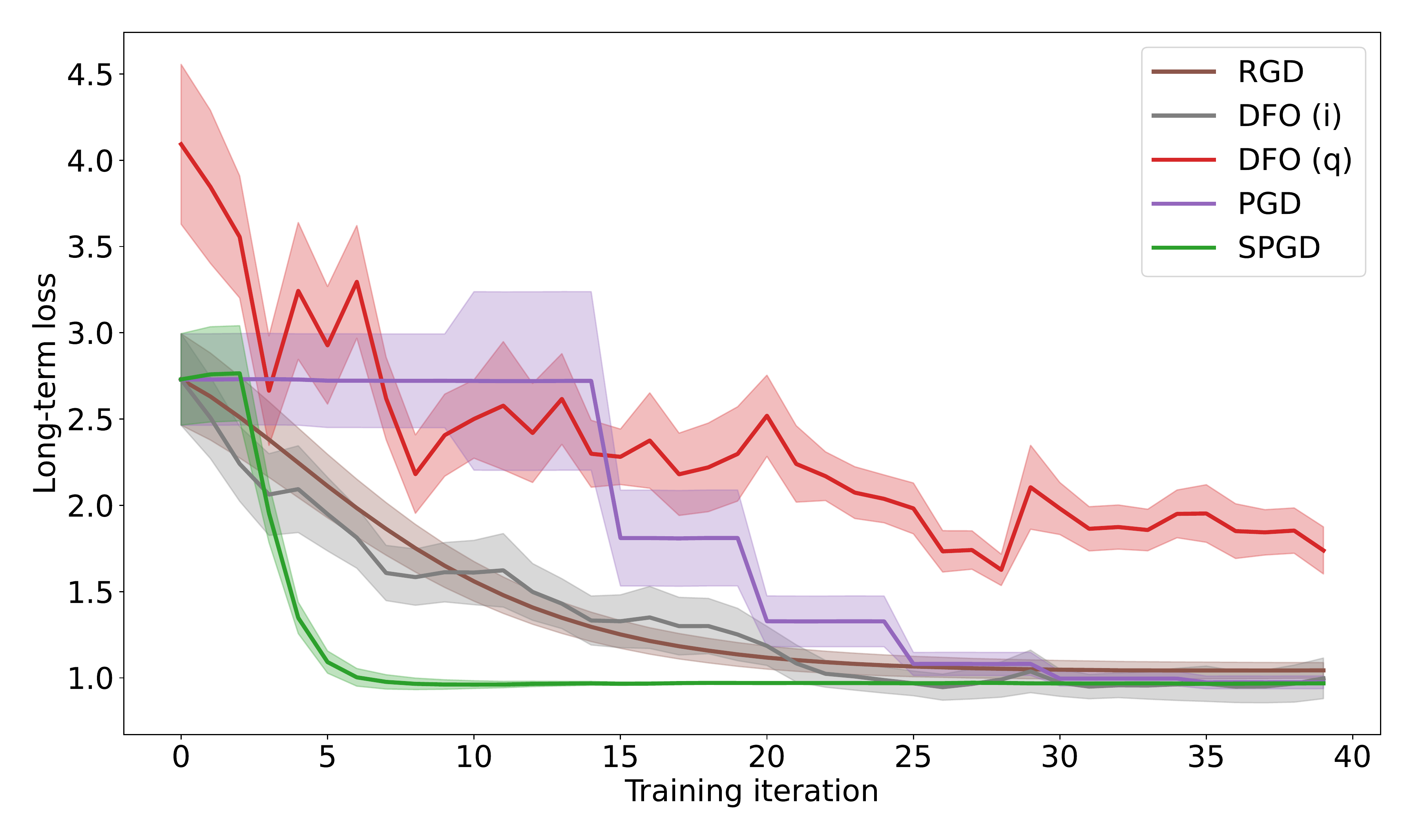}
    \caption{Results for spam classification. DFO (i) denotes the internal estimate of DFO, while DFO (q) denotes the models which are actually deployed by this algorithm. While DFO and PGD both find the optimal model, SPGD converges to it more rapidly. DFO must also deploy perturbed models (red curve) in order to find a good internal estimate. RGD converges to a stable point, resulting in $\sim10\%$ higher final loss.}
    \label{fig: classification}
\end{figure}

\subsection{Low-Dimensional Score} \label{sec: bottleneck expt}
Finally, we test SPGD's performance in the setting described in Section \ref{sec: bottleneck} where the distribution dynamics are constrained by a low-dimensional bottleneck. The point loss is $\ell(z; \th) = -z^\T \th + \frac{\lambda}{2}\|\th\|^2$, the score function is $s(\th, \mu) = \th^\T \mu$, and the stateful mean map is given by $m(\th, \mu) = \bar{m}(s(\th, \mu)) = (1 - \th^\T \mu)\mu_0$ for some fixed $\mu_0$. Under some restrictions on the model space $\Th$ and the parameter $\mu_0$, there exists a long-term distribution $\mu^*(\th)$. See Appendix \ref{appendix: experiment details} for a derivation.

Refer to Figure \ref{fig: bottleneck}.
BSPGD (Bottleneck SPGD) refers to SPGD where we account for the one-dimensional bottleneck in the dynamics. Both SPGD and BSPGD are able to find $\opt$, but by adapting the method to the one-dimensional score, BSPGD converges faster.

\begin{figure}
    \centering
    \includegraphics[width=.5\linewidth]{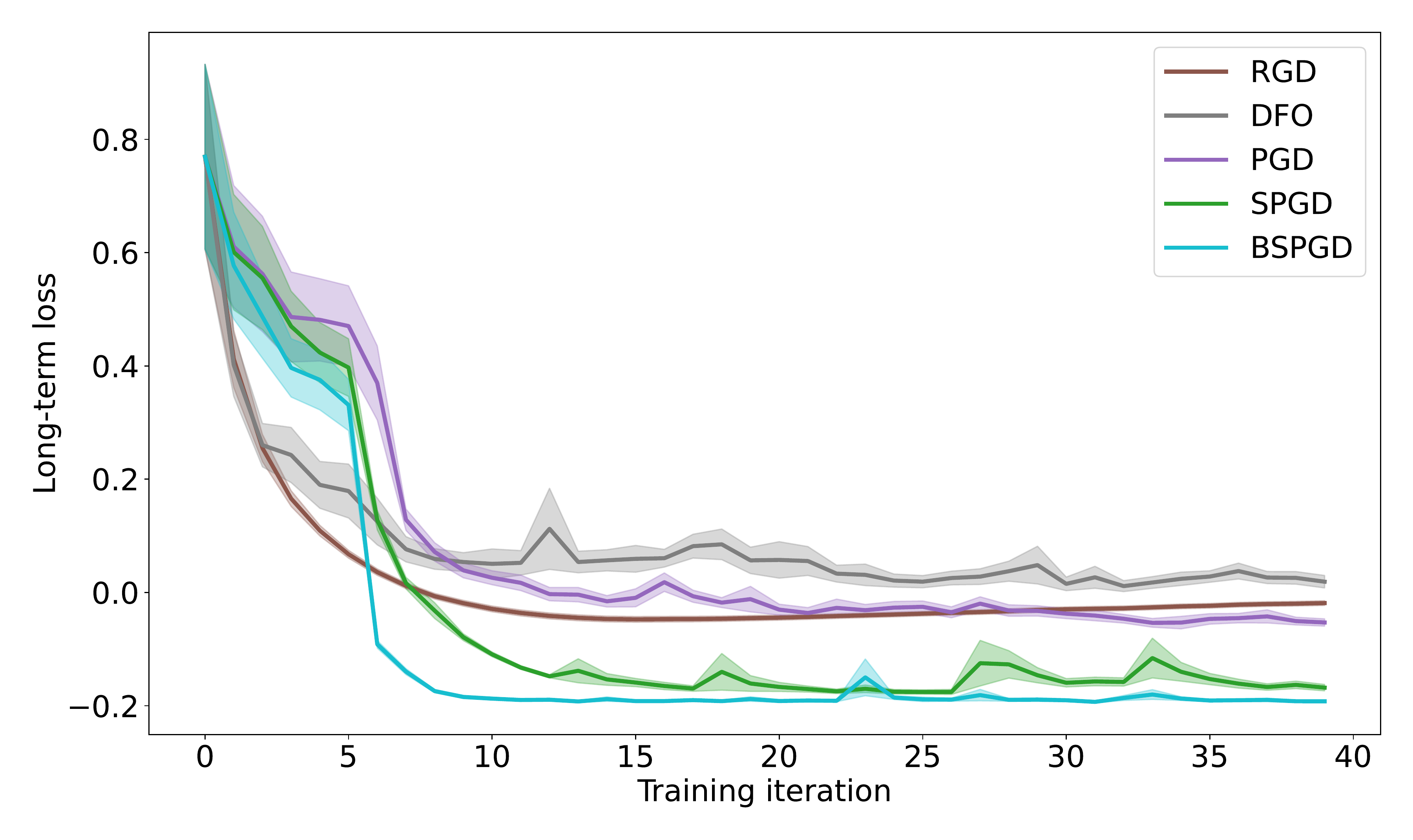}
    \caption{Performance of each method when the distribution shift depends only on a low-dimensional statistic. BSPGD denotes SPGD adapted to this setting. While vanilla SPGD outperforms both DFO and PGD, by taking into account the low-dimensional dependence, we can get even faster and more accurate convergence.}
    \label{fig: bottleneck}
\end{figure}


\section{CONCLUSION}
We considered the stateful performative setting and introduced Stateful PerfGD to optimize the long-term performative risk. We proved a convergence result for our method, and we verified empirically that our method is able to adapt to complicated stateful performative dynamics and find $\opt$, whereas existing methods not tailored to this situation prove ineffective.

While our work does require parametric data assumptions, optimizing the performative loss for a fully general distribution map $\D$ is intractible. The parametric framework still provides a great deal of modeling flexibility, leaving the entire toolkit of parametric statistics available to the user.
The assumption of a fixed long-term distribution may also appear restrictive, but for many types of performative effects---such as strategic behavior on the part of the modeled population---this assumption will indeed hold, as the agents will have no incentive to change their behavior in the face of a fixed model once their desired outcome has been reached. 

\subsection{Societal Impact}
The goal of optimization in the performative setting is to minimize the test loss. This is accomplished by choosing a model which is both accurate and induces a favorable data distribution, where ``favorable" is measured only with respect to the model's goal. When the population in question consists of people, this amounts to trying to induce these people to behave in a way which makes them easy to classify, which may not align with behaviors that benefit these people the most. Indeed, it has been observed that in some cases, such a procedure can maximize a certain measure of negative externality \citep{Jagadeesan2021-Strategic-Micro}. However, manipulation of the data distribution also has the capability to produce the opposite effect, i.e., inducing a data distribution which is advantageous both for the modeled population and the modeler. The distribution induced by the optimal model should also be studied to address these concerns.

\subsection{Future Work}
There are a number of interesting directions for future work. While minimizing the long-term performative risk is a sensible goal, other goals can also be considered---for instance, we can attempt to minimize the total loss incurred over the whole time horizon. In its current form, the problem is equivalent to a determinstic and highly structured Markov decision process, but relaxing some of the assumptions on the underlying MDP is of interest for improving the practical efficacy of this setting, and offers the potential for connections with reinforcement learning. Lastly, our current method works in the batch setting where we have enough samples to accurately estimate population-level quantities. Developing methods that can work in a stochastic/limited sample regime is also of interest.

\section*{Acknowledgements}
J.Z. is supported by NSF CAREER 1942926 and grants from the Silicon Valley Foundation and the Chan–Zuckerberg Initiative. L.Y. is supported by the Scientific Discovery through Advanced Computing program, and by the National Science Foundation DMS-1818449. Z.I. is supported by a Stanford Interdisciplinary Graduate Fellowship.

\nocite{*}
\bibliography{stateful}

\newpage
\onecolumn
\appendix
\section{DERIVATION OF STATEFUL PERFGD} \label{appendix: derivation}
The long-term performative loss is given by
$$ \L^*(\th_t) = \int \ell(z; \th_t) p(z; \mu^*(\th_t)) \, dz. $$
Its gradient is therefore given by
\begin{align*}
    \nb \L^*(\th_t) &= \int \nb_\th \ell(z; \th_t) \, p(z; \mu^*(\th_t)) \, dz + \int \ell(z; \th_t) \, \frac{d\mu^*}{d\th}^\T \nb_\mu p(z; \mu^*(\th_t)) \, dz.
\end{align*}
The general form of our gradient estimate arises by substituting $\mu_t$ for $\mu^*(\th_t)$ and $\widehat{\frac{dm^k}{d\th}}$ for $\frac{d\mu^*}{d\th}$.

The derivation for Algorithm \ref{alg: finite difference} is as follows. For each time $t$, let $\psi_t = [\th_t^\T, \, \mu_{t-1}^\T]^\T$, and define $m(\psi_t) = m(\th_t, \mu_{t-1}) = \mu_t$. By Taylor's theorem, we have
\begin{equation} \label{eq: taylor for m}
    m(\psi_s) - m(\psi_t) \approx  \frac{dm}{d\psi}\bigg|_{\psi_t} (\psi_s - \psi_t).
\end{equation}
then we can vectorize equation \eqref{eq: taylor for m} and obtain
$$ \Delta \mu \approx \frac{dm}{d\psi}\bigg|_{\psi_t} \Delta \psi \hspace{.15in} \Longrightarrow \hspace{.1in} \frac{dm}{d\psi}\bigg|_{\psi_t} \approx (\Delta \mu) (\Delta \psi)^\dagger. $$

The expression for $\widehat{\dmk}$ arises as follows. Observe that
\begin{align}
    \dth m^{(k)}(\th, \mu) &= \dth [m(\th, m^{(k-1)}(\th, \mu))] \nonumber \\[10pt]
    &= \p_1 m(\th, m^{(k-1)}(\th, \mu)) \label{eq: recursive dm^k} \\
    &+ \p_2 m(\th, m^{(k-1)}(\th, \mu)) \cdot \dth m^{(k-1)}(\th, \mu). \nonumber
\end{align}
Since $\mkm(\th, \mu)$ is unknown, as are the derivatives $\p_i m$ and  $\frac{d}{d\th} \mkm(\th, \mu)$, we simply substitute our ``best guess'' for each one. That is, we substitute $\mu$ for $\mkm(\th, \mu)$, $\wh{\p_i m}$ for $\p_i m$, and $\wh{\frac{d\mkm}{d\th}}$ for $\frac{d\mkm}{d\th}$. Thus we have
$$ \hdmk = \wh{\p_1 m(\th, \mu)} + \wh{\p_2 m(\th, \mu)} \cdot \wh{\frac{d\mkm}{d\th}},$$
with the base case $\wh{\frac{dm^{(0)}}{d\th}} = 0$ (the 0 matrix).
Let $\wh{\p_i m} = \wh{\p_i m(\th, \mu)}$. It can easily be shown via induction that
$$ \hdmk = (I + \wh{\p_2 m} + (\wh{\p_2 m})^2 + \cdots + (\wh{\p_2 m})^{k-1})(\wh{\p_1 m}). $$
Assuming that $\|\wh{\p_2 m}\| < 1$ (which we expect to hold since $\|\p_2 m\| < 1$), taking $k\rightarrow\infty$ yields
$$\lim_{k\rightarrow\infty} \hdmk = (I - \wh{\p_2 m})^{-1}(\wh{\p_1 m})$$
which is precisely the expression in \eqref{eq: lt jac}.

\section{PROOFS FOR \S \ref{sec: theory}} \label{appendix: proofs}

\subsection{Properties of the Gaussian distribution}
In the proofs which follow, we make use of several key properties of the Gaussian distribution. Some of the well-known facts we state without proof.

\begin{lemma} \label{thm: int grad p}
Let $p(z, \mu)$ be the probability density function of a $\N(\mu, \Sigma)$ random variable, where $\Sigma$ is a fixed covariance matrix. Then we have $ \int \| \nb_\mu p(z, \mu) \| \, dz \leq \| \Sigma^{-1} \| \sqrt{d}.$
\end{lemma}
\begin{proof}
We have
\begin{align}
    \int \| \nb_\mu p(z, \mu) \|\, dz &= \int \| \Sigma^{-1} (z - \mu) \| p(z, \mu) \, dz \nn \\
    &\leq \|\Sigma^{-1/2}\| \int \| \Sigma^{-1/2}(z-\mu) \| p(z, \mu) \, dz \nn \\
    &= \|\Sigma^{-1/2}\| \E_{z \sim \N(0, I_d)} [\|z\|] \label{eq: |grad p| 1} \\
    &\leq \|\Sigma^{-1/2}\|\sqrt{d} \label{eq: |grad p| 2}.
\end{align}
Here \eqref{eq: |grad p| 1} holds because $z \sim \N(\mu, \Sigma) \: \Rightarrow \: \Sigma^{-1/2}(z-\mu) \sim \N(0, I_d)$ and \eqref{eq: |grad p| 2} holds by the well-known inequality $\E_{z\sim \N(0, I_d)}[\|z\|] \leq \sqrt{d}$.
\end{proof}

\begin{lemma} \label{thm: |g| tail}
Let $\s_0^2 = \| \Sigma \|$. Then we have
$$\P_{z \sim \N(\mu, \Sigma)}( \| z - \mu \| \geq r + \s_0 \sqrt{d}) \leq c_1\exp\{ -c_2 r^2 / \s_0^2 \},$$
where $c_1$ is a constant which can depend on $d$ and $\Sigma$, and $c_2$ is a universal constant.
\end{lemma}

\begin{lemma} \label{thm: |g| whp}
Let $g \sim \N(0, \s^2 I) \in \R^d$. Then $\|g\| \leq \s(\sqrt{d} + c \sqrt{\log\g^{-1}})$ with probability at least $1 - \g$ for some universal constant $c$. By a union bound, this means that with probability at least $1-\g$, $\|g_t\| \leq \s(\sqrt{d} + c\sqrt{\log\frac{T}{\g}}) = \tilde{\O}(\s)$ for all $1 \leq t \leq T$.
\end{lemma}

\begin{lemma}[\cite{Anderson1955-Ineq}] \label{thm: |x + g|}
Suppose $X$ and $G$ are independent and $G \sim \N(0, \Sigma)$. Then for any $s>0$, we have $$\P(\|X + G\| \leq s) \leq \P(\|G\| \leq s).$$
\end{lemma}

\begin{lemma} \label{thm: bernstein |g| lb}
Let $g \sim \N(0, I_n)$, and let $A \in \R^{n\times n}$ have singular values $s_1 \geq \cdots \geq s_n$. Suppose that $s_k \geq c$. Then $\P(\|Ag\| \leq c\sqrt{k} - t) \leq 2 \exp( -c' \frac{t^2}{c^2} )$ for some universal constant $c'$.
\end{lemma}

\begin{proof}
Let $A = \sum_{i=1}^n s_i u_i v_i^\T$ be the SVD of $A$. Define $\gt_i = v_i^\T g$. Since the $v_i$ form an orthonormal basis, we have $\gt_i \iid \N(0,1)$. Next, observe that
\begin{align*}
    \| Ag \| &= \| \sum_{i=1^n} s_i (v_i^\T g) u_i \| \\
    &= \| (s_1 \gt_1, \: \ldots, \: s_n \gt_n)\| \\
    &\geq \| c (\gt_1, \: \ldots, \: \gt_k)\| \\
    &= c\|\gt\|
\end{align*}
where $\gt = (\gt_1, \: \ldots, \: \gt_k)^\T \sim \N(0, I_k)$. The result then follows directly from \citep{Vershynin2018-HDP}, Theorem 3.1.1.
\end{proof}

\begin{lemma} \label{thm: sv ineq}
Let $A \in \R^{n\times n}$ and let $s_1\geq \cdots \geq s_n$ be the singular values of $A$. Suppose that $\|A\|_F \leq cn$ and $\prod_{i=1}^n s_i \geq \beta^n$ for some $\beta > 0$. If $c_1 \sqrt{n} \leq k \leq c_2\sqrt{n}$ for some universal constants $c_1$ and $c_2$ then there exists $n_0 = \O((\log\frac{c}{\b})^2)$ such that for all $n \geq n_0$, $s_k \geq \frac\beta2$.
\end{lemma}

\begin{proof}
Let $k$ be arbitrary and suppose that $s_k < \frac\beta2$. Then we have
$$ \beta^n \leq \prod_{i=1}^n s_i \leq \left(\prod_{i=1}^k s_i \right) \left(\frac{\beta}{2}\right)^{n-k} \hspace{.15in} \Longrightarrow \hspace{.15in} \prod_{i=1}^k s_i \geq 2^n \left(\frac{\beta}{2}\right)^k. $$
On the other hand, we have
$$ \sum_{i=1}^k s_i \leq \sum_{i=1}^n s_i = \|A\|_* \leq \sqrt{n} \|A\|_F \leq cn^{3/2}. $$
A simple Lagrange multiplier argument implies that $\max \prod_{i=1}^k s_i$ s.t. $\sum_{i=1}^k s_i \leq C$ is $(C/k)^k$. Plugging in $C = cn^{3/2}$, we arrive at the inequality
\begin{equation} \label{eq: sv ineq} 2^n \left(\frac{\beta}{2}\right)^k \leq \prod_{i=1}^k s_i \leq \left(\frac{cn^{3/2}}{k}\right)^k \hspace{.15in} \Longrightarrow \hspace{.15in} 2^n \leq \left( \frac{ 2cn^{3/2} }{\beta k} \right)^k \hspace{.15in} \Longrightarrow \hspace{.15in} n \leq k \left( \log_2 \frac{2c}{\beta} + \log_2 \frac{n^{3/2}}{k} \right). \end{equation}
Now because $c_1\sqrt{n} \leq k \leq c_2 \sqrt{n}$, \eqref{eq: sv ineq} implies that
\begin{equation} \label{eq: sv ineq 2}
    n^{1/2} \leq c_2 (\log_2 \frac{2c}{\b} + \log_2 \frac{n}{c_1}) \leq c_2 \log_2 \frac{2c}{\b} + c_3 n^{1/4}
\end{equation}
for some universal constant $c_3$. Now inequality \eqref{eq: sv ineq 2} is quadratic in $n^{1/4}$, so applying the quadratic formula and simplifying, we see that it can only hold when $n \leq n_0$ for some $n_0 = \O((\log\frac{c}{\b})^2)$. This completes the proof.
\end{proof}

\subsection{Useful Properties of $m$ and $\mu^*$}
We will make use of the fact that $m^{(k+l)}(\th, \mu) = \mk(\th, m^{(l)}(\th, \mu))$ for any $k, l \geq 0$. This is a simple consequence of the fact that $\mk(\th, \mu)$ is the distribution parameters after $k$ deployments of $\th$ starting from $\mu$, and deploying $\th$ for $l$ steps followed by $k$ more deployments is the same as deploying $\th$ for $k+l$ steps. It can also be shown rigorously by a simple double inductive argument.

\begin{lemma} \label{thm: |d2mk|}
For any $k \geq 0$, we have $\| \p_2 \mk(\th, \mu) \| \leq \d^k$. In particular, since $0< \d < 1$, we have $\| \p_2 \mk(\th, \mu) \| \leq \d$ for all $k \geq 1$.
\end{lemma}

\begin{proof}
The claim is trivially true for $k=0$. Inducting on $k$, we have:
\begin{align*}
    \| \p_2 \mkp(\th, \mu) \| &= \| \frac{d}{d\mu} m(\th, \mk(\th, \mu)) \| \\
    &\leq \|\p_2 m(\th, \mk(\th,\mu)) \| \cdot \| \p_2 \mk(\th, \mu) \| \\
    &\leq \d \cdot \d^k.
\end{align*}
The above makes use of Assumption \ref{assumption: contraction} and the inductive hypothesis. This completes the proof.
\end{proof}

\begin{lemma} \label{thm: |d1mk|}
For any $k \geq 0$, we have $\| \p_1 \mk(\th, \mu) \| \leq \frac{B(1-\d^k)}{1-\d}$. In particular, since $0 < \d < 1$, we have $\| \p_1 \mk(\th, \mu) \| \leq \frac{B}{1-\d}$ for all $k \geq 0$.
\end{lemma}

\begin{proof}
The claim is trivially true for $k=0$. Inducting on $k$, we have:
\begin{align*}
    \| \p_1 \mkp(\th, \mu) \| &= \| \frac{d}{d\th} m(\th, \mk(\th, \mu)) \| \\
    &\leq \| \p_1 m(\th, \mk(\th, \mu)) \| + \| \p_2 m(\th, \mk(\th, \mu)) \| \| \p_1 \mk(\th, \mu) \| \\
    &\leq B + \d \frac{B(1-\d^k)}{1-\d} \\
    &= \frac{B(1-\d^{k+1})}{1-\d}.
\end{align*}
The above uses Assumptions \ref{assumption: m lip} and \ref{assumption: contraction} and the inductive hypothesis. This completes the proof.
\end{proof}

\begin{lemma} \label{thm: m^k to mu^*}
There exists a function $\mu^*(\th)$ such that $\lim_{k\rightarrow\infty} \mk(\th, \mu) = \mu^*(\th)$, independent of the starting parameters $\mu$. Furthermore, for any $\th, \mu$ we have
$$ \| \mk(\th, \mu) - \mu^*(\th) \| \leq D\d^k. $$
\end{lemma}

\begin{proof}
Let $\mu$ be arbitrary and consider the sequence $\mk(\th, \mu)$. We claim that this is a Cauchy sequence. WLOG let $k \leq l$. Then by Lemma \ref{thm: |d2mk|}, we have
\begin{align*}
    \| \mk(\th, \mu) - m^{(l)}(\th, \mu) \| &= \| \mk(\th, \mu) - \mk(\th, m^{(l-k)}(\th, \mu)) \| \\
    &\leq \d^k \| \mu - m^{(l-k)}(\th, \mu) \| \\
    &\leq D \d^k.
\end{align*}
Since $0 < \d < 1$, the sequence is Cauchy and therefore has a limit for any $\mu$. Furthermore, again by Lemma \ref{thm: |d2mk|}, we have
$$ \| \mk(\th, \mu) - \mk(\th, \mu') \| \leq \d^k \| \mu - \mu' \| \leq D \d^k, $$
which implies that the limit of these Cauchy sequences is independent of $\mu$. We can thus set $\mu^*(\th) = \lim_{k\rightarrow\infty} \mk(\th, \mu)$ for any $\mu$, and the above argument implies that $\mu^*$ is well-defined. It is also easy to see from this logic that $\mu^*(\th)$ must be a fixed point of $m(\th, \cdot)$.

Since $m^{(0)}(\th, \mu) = \mu$ and $\| \mu - \mu^*(\th) \| \leq D$ by definition of $D$, the claim holds for $k = 0$. We now induct and suppose the claim is true for arbitrary $k$. Then we have
\begin{align}
    \| \mkp(\th, \mu) - \mu^*(\th) \| &= \|m(\th, \mk(\th, \mu)) - m(\th, \mu^*(\th)) \| \label{eq: mk 1}\\
    &\leq \d \| \mk(\th, \mu) - \mu^*(\th) \| \label{eq: mk 2} \\
    &\leq \d \cdot D\d^k. \nonumber
\end{align}
Here \eqref{eq: mk 1} holds by the recursive definition of $\mkp$ and the fact that $\mu^*(\th)$ is a fixed point of $m(\th, \cdot)$, and \eqref{eq: mk 2} holds by Assumption \ref{assumption: contraction}. This completes the proof.
\end{proof}

\begin{lemma} \label{thm: mu* lip}
The long-term parameters $\mu^*(\th)$ are $\frac{B}{1-\d}$-Lipschitz in $\th$, and therefore $\dmus$ exists and we have $\| \dmus \| \leq \frac{B}{1-\d}$.
\end{lemma}

\begin{proof}
By Lemma \ref{thm: |d1mk|}, $\mk$ are uniformly Lipschitz in $\th$ with Lipschitz constant $B/(1-\d)$. Since $\mu^*$ is the limit of Lipschitz functions with Lipschitz constants uniformly bounded by $B/(1-\d)$, the result follows.
\end{proof}

\begin{lemma} \label{thm: dmk vs dmus}
Let $c = CD(1 + \frac{B}{1-\d})$ and $B' = \frac{B}{1-\d}$. Then $\| \dmk - \dmus \| \leq ck\d^{k-1} + B'\d^k = \O(k\d^k)$. In particular, $\lim_{k\rightarrow \infty} \| \dmk - \dmus \| = 0$.
\end{lemma}

\begin{proof}
In what follows, we will occasionally drop the dependence of $\mk(\th, \mu)$ on $\th$ and $\mu$ and the dependence of $\mu^*(\th)$ on $\th$ when this dependence is clear from context.

Since $\frac{dm^{(0)}}{d\th} = 0$ and $\| \dmus \| \leq B'$ by Lemma \ref{thm: mu* lip}, the claim holds for $k = 0$. We induct:
\begin{align*}
    \| \dmkp - \dmus \| &= \| \frac{d}{d\th} m(\th, \mk(\th, \mu)) - \frac{d}{d\th} m(\th, \mu^*(\th)) \| \\[5pt]
    &= \| \p_1 m(\th, \mk(\th, \mu)) + \p_2 m(\th, \mk(\th, \mu)) \dmk - \p_1 m(\th, \mu^*(\th)) - \p_2 m(\th, \mu^*(\th)) \dmus \| \\[5pt]
    &\leq \| \p_1 m(\th, \mk) - \p_1 m(\th, \mu^*) \| + \|\p_2 m(\th, \mk) - \p_2 m(\th, \mu^*)\| \|\dmk\| + \| \p_2 m(\th, \mu^*) \| \| \dmk - \dmus \| \\[5pt]
    &\leq C\| \mk - \mu^* \| + C\| \mk - \mu^*\| \frac{B}{1-\d} + \d \| \dmk - \dmus \| \\[5pt]
    &\leq CD(1 + \frac{B}{1-\d}) \d^k + \d \| \dmk - \dmus \| \\[5pt]
    &\leq c\d^k + \d(ck\d^{k-1} + B'\d^k) \\[5pt]
    &= c(k+1)\d^k + B'\d^{k+1}.
\end{align*}
The above makes use of Lemma \ref{thm: m^k to mu^*} and Assumption \ref{assumption: 2nd derivs}. The second part of the lemma then holds since $0 < \d < 1$.
\end{proof}

\begin{lemma} \label{thm: lt grad bounded}
The norm of the gradient of the long-term performative loss is bounded by a constant: $$\| \nb \L^*(\th) \| \leq G = \O\left(\frac{\lmax B \| \Sigma^{-1/2}\|\sqrt{d}}{1-\d}\right)$$ for all $\th$.
\end{lemma}

\begin{proof}
By Assumption \ref{assumption: loss} and Lemma \ref{thm: mu* lip}, we have
\begin{align*}
    \|\nb \L^*(\th)\| &\leq \int \|\nb_\th \ell(z; \th)\| \, p(z; \mu^*(\th)) \, dz + \int |\ell(z; \th)| \, \|\frac{d\mu^*}{d\th}^\T\| \, \|\nb_\mu p(z; \mu^*(\th))\| \, dz \\
    &\leq \lmax + \lmax \frac{B}{1-\d} \int \| \nb_\mu p(z, \mu^*(\th)) \| \, dz \\
    &\leq G = \O\left(\frac{\lmax B \| \Sigma^{-1/2}\|\sqrt{d}}{1-\d}\right).
\end{align*}
The last line follows from Lemma \ref{thm: int grad p}.
\end{proof}

We remark that, by a proof similar to the preceding, the bound on $\| \nb^2 \L^*(\th) \|$ in Assumption \ref{assumption: 2nd derivs} can be replaced with a bound on the second derivatives of $\ell$. This constant will not appear in the leading order terms of our final bounds, so we opt for the simpler route of just assuming a priori that $\L^*$ has a bounded Hessian.

\blockcomment{
\begin{lemma} \label{thm: mus hess bounded}
$\| \frac{d^2\mu^*}{d\th^2} \| \leq ?$
\end{lemma}

\begin{proof}
We have that $\p_i \p_j m$ is a bilinear map $\R^d \times \R^d \rightarrow \R^d$.
\begin{align*}
    \frac{d^2}{d\th^2} \mkp(\th, \mu) &= \frac{d^2}{d\th^2} m(\th, \mk(\th, \mu)) \\
    &= \frac{d}{d\th} [ \p_1 m(\th, \mk) + \p_2 m(\th, \mk) \p_1 \mk ] \\
    &= \p_1^2 m(\th, \mk) + \p_2 \p_1 m(\th, \mk) \frac{d\mk}{d\th} \\
    &+ \left(\p_1 \p_2 m(\th, \mk) + \p_2^2m(\th,\mk) \frac{d\mk}{d\th}\right)\p_1\mk \\
    &=
\end{align*}
\end{proof}

\begin{lemma} \label{thm: lt hess bounded}
The Hessian of the long-term performative loss is bounded: $\| \nb^2 \L^*(\th) \| \leq C' = \O(1)$ for all $\th$.
\end{lemma}

\begin{proof}
\begin{align*}
    \| \nb^2 \L^*(\th) \| &\leq \int \| \nb^2_\th \ell(z; \th) \| p(z; \mu^*(\th)) \, dz + 2 \int \| \nb_\th\ell(z; \th) \| \| \dmus \| \| \nb_\mu p(z; \mu^*(\th)) \| \, dz \\
    &+ \int |\ell(z; \th)| (\| \frac{d^2 \mu^*}{d\th^2} \| \| \nb_\mu p(z, \mu^*(\th)) + \| \dmus \|^2 \| \nb^2_\mu p(z, \mu^*(\th)) \|)\, dz \\[5pt]
    &\leq 
\end{align*}

\end{proof}
}

\subsection{Approximation Results}
Throughout, we will assume that $T = \O(\be^{-c})$ for some universal constant $c > 0$. (We can always stop the optimization procedure early if $T$ is larger than this.) This implies that $\log T = o(\n^{-c'}), \: o(\s^{-c})$ for any positive constant $c'$.

In our bounds, we will keep track of terms which are of leading order as $\n, \, \s, \, \be,\, T^{-1} \rightarrow 0$. Given our eventual choices of $\n$ and $\s$, we will always have $\be = o(\n)$, $\n = o(\s)$, and $\s = o(1)$. We will also track the problem-dependent constants (e.g., $B$, $C$, $D$ from the assumptions, the dimension $d$, etc.) which form coefficients for these leading order terms, but we still consider these as constants and therefore drop terms which are high order in $\n, \, \s, \, \be, \, T^{-1}$ but with worse dependence on the aforementioned constants. We also remark that we have not attempted to optimize our bounds with respect to these constants, and the dependence on them is likely not tight. Lastly, since the constant $G$ defined in Lemma \ref{thm: lt grad bounded} appears frequently, we will make use of it rather than repeatedly writing $\frac{\lmax B \|\Sigma^{-1/2}\| \sqrt{d}}{1-\d}$, but it should be noted that $G$ can in fact be replaced by constants which exist a priori by the assumptions.

The overall structure of these proofs is inductive in nature. That is, we assume some conditions on the optimization trajectory so far---namely, bounds on the errors of various estimators---, and show that these properties continue to hold at the next step of the optimization.

We begin by showing that, after an initialization or ``burn-in" period, the observed population means will be close to their equilibrium values. (In practice, the initialization can be quite short.) For ease of notation, we will always denote the $\th$ update steps as $\hnb \L^*_t + g_t$, though for the initialization phase we will just take $\hnb \L^*_t = 0$. (That is, we initialize by updating $\th_t$ by random Gaussian perturbations.)

\begin{lemma} \label{thm: mu_t vs mu*_t}
Suppose that $t \geq \log\frac{1}{\n}$, $\s = o(1/\sqrt{\log\frac{T}{\g}})$, and $\| \hnb \L^*_s \| \leq cG$ for each $s < t$. Then we have
$$ \| \mu_t - \mu^*(\th_t) \| = \O(BG(\log\frac{1}{\n})^2\n), $$
with probability at least $1-\g$ simultaneously for all $t$.
\end{lemma}

\begin{proof}
We claim that
\begin{equation} \label{eq: mut* 1}
    \| \mu_t - \mu^*(\th_t) \| \leq \d B \sum_{l = 1}^{k - 1} \| \th_{t-l} - \th_t \| + \| \mk(\th_t, \mu_{t-k}) - \mu^*(\th_t) \|
\end{equation}
for any $1 \leq k \leq t$. For $k = 1$, \eqref{eq: mut* 1} is just the statement $\| \mu_t - \mu^*(\th_t) \| \leq \| m^{(1)}(\th_t, \mu_{t-1}) - \mu^*(\th_t) \|$, which is true since $\mu_t = m^{(1)}(\th_t, \mu_{t-1})$ and thus the LHS and RHS are equal. Now we induct on $k$:
\begin{align}
    \| \mk(\th_t, \mu_{t-k}) - \mu^*(\th_t) \| &= \| \mk(\th_t, m(\th_{t-k}, \mu_{t-k-1})) - \mu^*(\th_t) \| \nonumber \\
    &\leq \| \mk(\th_t, m(\th_{t-k}, \mu_{t-k-1})) - \mk(\th_t, m(\th_t, \mu_{t-k-1})) \| + \| \mk(\th_t, m(\th_t, \mu_{t-k-1})) - \mu^*(\th_t) \| \nonumber \\
    &\leq \d \| m(\th_{t-k}, \mu_{t-k-1}) - m(\th_t, \mu_{t-k-1}) \| + \| \mk(\th_t, m(\th_t, \mu_{t-k-1})) - \mu^*(\th_t) \| \label{eq: mut* 2} \\
    &\leq \d B \|\th_{t-k} - \th_t\| + \| \mk(\th_t, m(\th_t, \mu_{t-k-1})) - \mu^*(\th_t) \label{eq: mut* 3} \|.
\end{align}
Here \eqref{eq: mut* 2} follows from Lemma \ref{thm: |d2mk|} and \eqref{eq: mut* 3} uses Assumption \ref{assumption: m lip}. If we use the fact that $\mkp(\th, \mu_{t-k-1}) = \mk(\th_t, m(\th_t, \mu_{t-k-1}))$ (this is simply unrolling the recursive definition for $\mkp$ from the inside out instead of outside in) and plug \eqref{eq: mut* 3} into the inductive hypothesis, we complete the induction and \eqref{eq: mut* 1} holds for all $k \leq t$.

Next, for $l\leq k$, observe that
\begin{align}
    \| \th_{t-l} - \th_t \| &= \n \| \hnb \L^*_{t_l} + g_{t-l} + \cdots + \hnb \L^*_{t-1} + g_{t-1} \| \nn \\
    &\leq \n \sum_{i=1}^l \| \hnb \L^*_{t-i} \| + \| g_{t-i} \| \nn \\
    &\leq \n \cdot l \cdot (cG + o(1)) \label{eq: use |g| bound} \\
    &\leq c'' G k \n. \label{eq: th dist}
\end{align}
Here \eqref{eq: use |g| bound} holds since we have assumed the the high-probability guarantee of Lemma \ref{thm: |g| whp} holds for all $t$.
Plugging this inequality into \eqref{eq: mut* 1}, we have that
\begin{align*}
    \| \mu_t - \mu^*(\th_t) \| &\leq \d B \cdot (k-1) \cdot c'' G k\n + D\d^k \\
    &= \O(BGk^2 \n + D\d^k)
\end{align*}
where we have also used Lemma \ref{thm: m^k to mu^*} to bound $\| \mk(\th_t, \mu_{t-k}) - \mu^*(\th_t) \|$. Setting $k = \log\frac1\n$ (which is valid since $t \geq \log\frac1\n$), we obtain
\begin{equation*}
    \| \mu_t - \mu^*(\th_t) \| = \O(BG(\log\frac{1}{\n})^2 \n) = \tilde{\O}(\n).
\end{equation*}
\end{proof}

\fderr*

\begin{proof}
We begin by decomposing $\Delta \mu$ and $\Delta \psi$.
\begin{align*}
    \Delta \mu &= \begin{bmatrix} \hmu_{t-1} - \hmu_t & \cdots & \hmu_{t-H} - \hmu_t \end{bmatrix} \\[5pt]
    &= \underbrace{\begin{bmatrix} \mu_{t-1} - \mu_t & \cdots & \mu_{t-H} - \mu_t \end{bmatrix}}_{\overline{\Delta \mu}} + \underbrace{\begin{bmatrix} \e_{t-1} - \e_t & \cdots & \e_{t-H} - \e_t \end{bmatrix}}_{\mathrm{err}_{\mu}} \\[10pt]
    \Delta \psi &= \begin{bmatrix} \th_{t-1} - \th_t & \cdots & \th_{t-H} - \th_t \\ \hmu_{t-2} - \hmu_{t-1} & \cdots & \hmu_{t-H-1} - \hmu_{t-1} \end{bmatrix} \\[5pt]
    &= \underbrace{\begin{bmatrix} \th_{t-1} - \th_t & \cdots & \th_{t-H} - \th_t \\ \mu_{t-2} - \mu_{t-1} & \cdots & \mu_{t-H-1} - \mu_{t-1} \end{bmatrix}}_{\overline{\Delta \psi}} + \underbrace{\begin{bmatrix} 0 & \cdots & 0 \\ \e_{t-2} - \e_{t-1} & \cdots & \e_{t-H-1} - \e_{t-1} \end{bmatrix}}_{\mathrm{err}_\psi}
\end{align*}
Using this expression, we can rewrite $\wh{\frac{dm}{d\psi}}$:
\begin{equation*}
    \wh{\frac{dm}{d\psi}} = (\bardmu)(\bardpsi)^\dagger + (\errmu)(\bardpsi)^\dagger + (\Delta \mu)[ (\bardpsi + \errpsi)^\dagger - (\bardpsi)^\dagger].
\end{equation*}
This allows us to decompose the error of $\wh{\frac{dm}{d\psi}}$ as
\begin{equation} \label{eq: m jac err 1}
    \| \frac{dm}{d\psi} - (\Delta \mu)(\Delta \psi)^\dagger \| \leq \| \frac{dm}{d\psi} - (\bardmu)(\bardpsi)^\dagger \| + \|\errmu\| \| \bardpsi^\dagger \| + \underbrace{\| \Delta \mu \| \|(\bardpsi + \errpsi)^\dagger - (\bardpsi)^\dagger\|}_{(\star)}. 
\end{equation}
Before we begin, let us decompose $(\star)$ further. By \citep{Wedin1973-Perturbation}, if $\Delta \psi$ and $\bardpsi$ are both full rank, then
\begin{equation} \label{eq: m jac err 2}
\|(\bardpsi + \errpsi)^\dagger - (\bardpsi)^\dagger\| \leq \sqrt{2} \|(\bardpsi + \errpsi)^\dagger\| \|\bardpsi^\dagger\| \|\errpsi\|.
\end{equation}
By Weyl's inequality for singular values, we also have
\begin{align}
    \|(\bardpsi + \errpsi)^\dagger\| &= \frac{1}{\smin(\bardpsi + \errpsi)} \nonumber \\[5pt]
    &\leq \frac{1}{\smin(\bardpsi) - \s_{\max}(\errpsi)} \nonumber \\[5pt]
    &= \frac{1}{\frac{1}{\|\bardpsi^\dagger\|} - \| \errpsi \|} \nonumber \\[5pt]
    &= \frac{\| \bardpsi^\dagger \|}{1 - \| \bardpsi^\dagger \| \| \errpsi \|}. \label{eq: m jac err 3}
\end{align}
Combining \eqref{eq: m jac err 2} and \eqref{eq: m jac err 3} and substituting them into \eqref{eq: m jac err 1}, we obtain
\begin{equation} \label{eq: I + II + III}
    \| \frac{dm}{d\psi} - (\Delta \mu)(\Delta \psi)^\dagger \| \leq \underbrace{\| \frac{dm}{d\psi} - (\bardmu)(\bardpsi)^\dagger}_{\textrm{(I)}} \| + \underbrace{\|\errmu\| \| \bardpsi^\dagger \|}_{\textrm{(II)}} + \sqrt{2}\underbrace{\| \Delta \mu \| \frac{\| \bardpsi^\dagger \|^2}{1 - \| \bardpsi^\dagger \| \| \errpsi \|}\| \errpsi \|}_{\textrm{(III)}}.
\end{equation}
Let us address (I) first. Recall that $\dmpsi$ refers to the Jacobian of $m$ evaluated at $(\th_t, \mu_{t-1})$. By Taylor's theorem, for any $1 \leq i \leq H$, we have
\begin{equation} \label{eq: I 1}
    \mu_{t-i} - \mu_t = \dmpsi (\psi_{t-i} - \psi_t) + \nb^2 m(\xi_i)[\psi_{t-i} - \psi_t, \, \psi_{t-i} - \psi_t].
\end{equation}
Here $\psi_{t-i} = (\th_{t-i}^\T, \mu_{t-1-i}^\T)^\T$ and $\nb^2 m(\xi_i)$ denotes the tensor of second derivatives of $m$ evaluated at some point $\xi_i$ specified by Taylor's theorem. For each $\xi_i$, $\nb^2 m(\xi_i)$ is a bilinear map from $\R^{2d} \times \R^{2d} \rightarrow \R^d$, and $\nb^2 m(\xi_i)[\psi_{t-i} - \psi_t, \, \psi_{t-i} - \psi_t]$ denotes evaluation of this map at the inputs specified in the brackets. 
Define $\taylor_i = \nb^2 m(\xi_i)[\psi_{t-i} - \psi_t, \, \psi_{t-i} - \psi_t] $. By \eqref{eq: I 1}, we have
\begin{align}
    \bardmu &= \begin{bmatrix} \dmpsi (\psi_{t-1} - \psi_t) + \taylor_1 & \cdots & \dmpsi (\psi_{t-H} - \psi_t) + \taylor_H \end{bmatrix} \nonumber \\[5pt]
    &= \dmpsi \begin{bmatrix} \psi_{t-1} - \psi_t & \cdots & \psi_{t-H} - \psi_t \end{bmatrix} + \begin{bmatrix} \taylor_1 & \cdots & \taylor_H \end{bmatrix} \nonumber \\[5pt]
    &= \dmpsi \bardpsi + \underbrace{\begin{bmatrix} \taylor_1 & \cdots & \taylor_H \end{bmatrix}}_{\Taylor}. \label{eq: delta mu decomp}
\end{align}
Assume for the moment that $\bardpsi$ has full rank; we will prove this later. Since we have chosen $H \geq 2d$ and $\bardpsi \in \R^{2d \times H}$, this implies that $\bardpsi$ has a right inverse. Combining this fact with \eqref{eq: delta mu decomp}, we have
\begin{align}
    \textrm{(I)} &= \| \dmpsi - (\bardmu)(\bardpsi)^\dagger \| \nonumber \\[5pt]
    &= \| \dmpsi - (\dmpsi + \Taylor (\bardpsi)^\dagger) \| \nonumber \\[5pt]
    &\leq \| \Taylor \| \| \bardpsi^\dagger \|. \label{eq: I 2}
\end{align}
We now bound $\| \Taylor \|$. Because the operator norm $\| \nb^2 m \| \leq C$ by Assumption \ref{assumption: 2nd derivs}, we have
\begin{equation} \label{eq: m taylor err}
    \| \taylor_i \| = \| \nb^2 m(\xi_i)[\psi_{t-i} - \psi_t, \, \psi_{t-i} - \psi_t] \| \leq C \|\psi_{t-i} - \psi_t \|^2.
\end{equation}
It follows that
\begin{align}
    \| \Taylor \| &\leq \| \Taylor \|_F \nn \\[5pt]
    &= \sqrt{\sum_{i=1}^H \| \taylor_i \|^2 } \nn \\[5pt]
    &\leq \sqrt{\sum_{i=1}^H (C \| \psi_{t-i} - \psi_t \|^2)^2 }. \label{eq: ETaylor 1}
\end{align}
By definition of $\psi_{t-i}$ and $\psi_t$, we have
\begin{align}
    \| \psi_{t-i} - \psi_t \|^2 &= \left\| \begin{bmatrix} \th_{t-i} - \th_t \\ \mu_{t-1-i} - \mu_{t-1} \end{bmatrix} \right\|^2 \nn \\[5pt]
    &= \| \th_{t-i} - \th_t \|^2 + \| \mu_{t-1-i} - \mu_{t-1} \|^2 \nn \\[5pt]
    &\leq \O((GH\n)^2) + (\| \mu_{t-1-i} - \mu^*(\th_{t-1-i}) \| + \| \mu^*(\th_{t-1-i}) - \mu^*(\th_{t-1}) \| + \| \mu^*(\th_{t-1}) - \mu_{t-1} \|)^2 \label{eq: ETaylor 2} \\[5pt]
    &\leq \O(G^2H^2 \n^2) + ( \O(BG(\log\frac1\n)^2 \n) + \frac{B}{1-\d} \| \th_{t-1-i} - \th_{t-1} \| + \O(BG(\log\frac1\n)^2 \n) )^2 \label{eq: ETaylor 3} \\[5pt]
    &\leq \O(G^2 H^2 \n^2) + [\O( BG(\log\frac1\n)^2 \n + \frac{B}{1-\d}GH\n )]^2 \label{eq: ETaylor 4} \\[5pt]
    &= \O(B^2 G^2 (\log\frac1\n)^4 \n^2). \label{eq: ETaylor 5}
\end{align}
Inequality \eqref{eq: ETaylor 2} holds by the logic from \eqref{eq: th dist}, as well as by splitting $\| \mu_{t-1-i} - \mu_{t-1} \|$ with the triangle inequality. Inequality \eqref{eq: ETaylor 3} holds by Lemmas \ref{thm: mu* lip} and \ref{thm: mu_t vs mu*_t}. Finally, inequality \eqref{eq: ETaylor 4} again holds by the logic from \eqref{eq: th dist}.

Plugging the result of \eqref{eq: ETaylor 5} into \eqref{eq: ETaylor 1}, we have
\begin{align}
    \| \Taylor \| &\leq C \sqrt{\sum_{i=1}^H [\O( B^2 G^2(\log\frac1\n)^4\n^2 )]^2} \nn \\[5pt]
    &\leq C\sqrt{H \cdot [\O( B^2 G^2 (\log\frac1\n)^4 \n^2)]^2 } \nn \\[5pt]
    &= \O( C\sqrt{H} B^2 G^2 (\log\frac1\n)^4 \n^2 ) \nn \\[5pt]
    &= \tilde{\O}(\n^2). \label{eq: ETaylor final}
\end{align}
Combining this with $\eqref{eq: I 2}$ yields that
\begin{equation} \label{eq: I 3}
    \textrm{(I)} = \O( C\sqrt{H} B^2 G^2 (\log\frac1\n)^4 \n^2 ) \| \bardpsi^\dagger \|.
\end{equation}

Next, observe that since $\|\e_{t-i}\| \leq \be$, we have
\begin{equation} \label{eq: errmu errpsi}
    \| \errmu \| \leq \| \errmu \|_F = \O(\sqrt{H} \be), \hspace{.25in} \| \errpsi \| \leq \| \errpsi \|_F = \O(\sqrt{H} \be).
\end{equation}

We now turn our attention back to (III); in particular we will bound $\| \Delta \mu \|$. By the definition of $\bardmu$, we have
\begin{align}
    \| \Delta \mu \| &\leq \| \bardmu \| + \| \errmu \| \nn \\[5pt]
    &\leq \| \dmpsi \| \| \bardpsi \| + \| \Taylor \| + \| \errmu \| \label{eq: III 1} \\[5pt]
    &= \O\left( B\| \bardpsi \|  + C\sqrt{H} B^2 G^2 (\log\frac{1}{\n})^4 \n^2  + \sqrt{H} \be \right). \label{eq: III 2}
\end{align}
Here \eqref{eq: III 1} follows from equation \eqref{eq: delta mu decomp}. Equation \eqref{eq: III 2} follows from \eqref{eq: ETaylor final}, \eqref{eq: errmu errpsi}, and Assumptions \ref{assumption: m lip} and \ref{assumption: contraction}.

To bound $\| \bardpsi \|$, we make use of \eqref{eq: ETaylor 5}. We have
\begin{align}
    \| \bardpsi \| &\leq \| \bardpsi \|_F \nn \\[5pt]
    &= \sqrt{\sum_{i=1}^H \| \psi_{t-i} - \psi_t \|^2} \nn \\[5pt]
    &\leq \sqrt{ H \cdot \O( B^2 G^2 (\log\frac1\n)^4 \n^2 ) } \label{eq: |dpsi| 1} \\[5pt]
    &= \O( BG \sqrt{H} (\log\frac1\n)^2 \n ) \nn \\[5pt]
    &= \tilde{\O}(\n). \label{eq: |dpsi| 2}
\end{align}
Inequality \eqref{eq: |dpsi| 1} follows from \eqref{eq: ETaylor 5}. Plugging \eqref{eq: |dpsi| 2} into \eqref{eq: III 2}, we find that
\begin{align}
    \| \Delta \mu \| &= \O( B^2 G\sqrt{H} (\log\frac{1}{\n})^2 \n + \sqrt{H}\be) \nn \\[5pt]
    \Rightarrow \textrm{(III)} &= \O(B^2 G\sqrt{H} (\log\frac{1}{\n})^2 \n) \cdot \frac{\| \bardpsi^\dagger \|^2}{1 - \| \bardpsi^\dagger \| \be } \cdot \O(\sqrt{H}\be) \label{eq: III 3} \\[5pt]
    &= \O(B^2 G H (\log\frac{1}{\n})^2 \n\be \| \bardpsi^\dagger \|^2 ). \label{eq: III final}
\end{align}
Equation \eqref{eq: III 3} uses the definition of (III) and the bound on $\| \Delta \mu \|$ (which we have reduced using that fact that $\be = o(\n)$), as well as inequality \eqref{eq: errmu errpsi}. Equation \eqref{eq: III final} holds under the assumption that $\| \bardpsi^\dagger \| \|\errpsi\| \leq \frac12$, which we will show holds given the final choice of $\n$ and $\s$.

Let us combine what we have shown so far to reduce the expression \eqref{eq: I + II + III}:
\begin{equation} \label{eq: reduced m jac err bound}
    \| \dmpsi - (\Delta \mu)(\Delta \psi)^\dagger \| = \O\left(C\sqrt{H} B^2 G^2 (\log\frac{1}{\n})^4\n^2 \| \bardpsi^\dagger \| + \sqrt{H}\be \| \bardpsi^\dagger \| + B^2 G H (\log\frac{1}{\n})^2 \n\be \| \bardpsi^\dagger \|^2 \right). 
\end{equation}
We have thus reduced the problem to showing that $\bardpsi$ has full rank and bounding $\| \bardpsi^\dagger \|$ from above. By Lemma \ref{thm: |dpsi^dag|} below, $\| \bardpsi^\dagger \| = \O(\frac{1}{\n\s})$. Plugging this into \eqref{eq: reduced m jac err bound}, we arrive at
\begin{align}
    \| \dmpsi - (\Delta \mu)(\Delta \psi)^\dagger \| &= \O\left(C\sqrt{H} B^2 G^2 (\log\frac{1}{\n})^4\n^2 \frac{1}{\n\s} + \sqrt{H}\be \frac{1}{\n\s} + B^2 G H (\log\frac{1}{\n})^2 \n\be \frac{1}{\n^2 \s^2} \right) \nn \\[5pt]
    &= \O\left(C\sqrt{H} B^2 G^2 (\log\frac{1}{\n})^4 \frac{\n}{\s} + B^2 G H (\log\frac{1}{\n})^2 \frac{\be}{\n \s^2} \right) \equiv \Em. \nn
\end{align}
This completes the proof.
\end{proof}

\begin{lemma} \label{thm: |dpsi^dag|}
Suppose that $\| \hnb \L^*(\th_{t-i}) - \nb \L^*(\th_{t-i}) \| \leq \Eg$, where
\begin{equation} \label{eq: Eg 1}
\Eg = \O\left(\frac{C B^3 G^2 \sqrt{Hd} \|\Sigma^{-1/2}\|}{(1-\d)^2} (\log\frac{1}{\n})^4 \frac{\n}{\s} + \frac{B^3 G H\sqrt{d} \|\Sigma^{-1/2}\|}{(1-\d)^2} (\log\frac{1}{\n})^2 \frac{\be}{\n \s^2} \right) = \tilde{\O}(\frac{\n}{\s} + \frac{\be}{\n\s^2}).
\end{equation}
Furthermore, assume that
$$ \s \geq 2c'\frac{BH^{3/2}\Eg}{1-\d} \hspace{.25in} \textrm{and} \hspace{.25in} \s = o(1)$$
where $c'$ is an absolute constant to be specified later.
Then with probability at least $1 - \frac{3\g}{T}$, $\bardpsi$ has full rank and $\|\bardpsi^\dagger\| = \O(\frac{1}{\n\s})$ as long as $H = \Theta(\frac{B^8 d^4}{\alpha^8(1-\d)^8}(\log\frac{T}{\g})^4)$.
\end{lemma}

\begin{proof}
Define $e_{t-i} = \hnb \L^*(\th_{t-i}) - \nb \L^*(\th_{t-i})$ for $i=1,\ldots,H$, so $\| e_{t-i} \| \leq \Eg$ for $i=1,\ldots,H$. (Note: If $t-i \leq \log\frac{1}{\n}$ so that step $t-i$ was during the intialization phase, then $\hnb \L^*(\th_{t-i})$ and $\nb \L^*(\th_{t-i})$ are both replaced with $0$ and the error is $0$ for these steps. The logic that follows can trivially be extended to this case.)

\paragraph{Claim 1:} We have
$$\th_{t-H+i} = \bar{\th}_{t-H+i} - \n \underbrace{\sum_{j=0}^{i-1} g_{t-H+j}}_{G_{t-H+i}} - \n \underbrace{\sum_{j=0}^{i-1} e_{t- H + j}}_{E_{t-H+i}} + \n^2 S_{t-H+i},$$
where $\bar{\th}_{t-H+i}$ is independent of all of the Gaussian perturbations $g_{t-H+j}$, $j=0,\ldots,H-1$, and $S_{t-H+i}$ is defined recursively by
$$S_{t-H+i} = \sum_{j=0}^{i-1} \nb^2 \L^*(\zeta_{t-H+j}) (G_{t-H+j} + E_{t-H+j} - \n S_{t-H+j}) $$
for some collection of vectors $\zeta_{t-H+j}$ in $\R^d$.

\paragraph{Proof of Claim 1:} The claim holds trivially when $i=0$, so we induct on $i$. We have:
\begin{align}
    \th_{t-H+i+1} &= \th_{t-H+i} - \n (\nb \L^*(\th_{t-H+i}) + g_{t-H+i} + e_{t-H+i}) \nn \\[5pt]
    &= \th_{t-H+i} - \n (g_{t-H+i} + e_{t-H+i}) - \n \nb \L^*(\bar{\th}_{t-H+i} - \n G_{t-H+i} - \n E_{t-H+i} + \n^2 S_{t-H+i}) \nn \\[5pt]
    &= \th_{t-H+i} - \n (g_{t-H+i} + e_{t-H+i}) \nn \\
    &\phantom{= } - \n ( \nb\L^*(\bar{\th}_{t-H+i}) - \n \nb^2\L^*(\zeta_{t-H+i})(G_{t-H+i} + E_{t-H+i} - \n S_{t-H+i})) \label{eq: taylor expand grad L*} \\[5pt]
    &= \bar{\th}_{t-H+i} - \n \nb\L^*(\bar{\th}_{t-H+i}) - \n (G_{t-H+i} + g_{t-H+i}) - \n (E_{t-H+i} + e_{t-H+i}) \nn \\
    &\phantom{= } + \n^2 (S_{t-H+i} + \nb^2 \L^*(\zeta_{t-H+i}) (G_{t-H+i} + E_{t-H+i} - \n S_{t-H+i}). \nn
\end{align}
Equation \eqref{eq: taylor expand grad L*} follows by Taylor expanding $\nb \L^*$ about $\bar{\th}_{t-H+i}$. We observe that
\begin{align*}
    G_{t-H+i} + g_{t-H+i} &= \sum_{j=0}^{i-1} g_{t-H+j} + g_{t-H+i} \\
    &= \sum_{j=0}^i g_{t-H+j} \\
    &= G_{t-H+i+1}.
\end{align*}
Similarly, we have
\begin{equation*}
    E_{t-H+i} + e_{t-H+i} = E_{t-H+i+1}
\end{equation*}
\begin{equation*}
    S_{t-H+i} + \nb^2 \L^*(\zeta_{t-H+i})(G_{t-H+i} + E_{t-H+i} - \n S_{t-H+i}) = S_{t-H+i+1},
\end{equation*}
which completes the induction and proves Claim 1.

Before we proceed, we first bound $\|E_{t-H+i}\|$ and $\|S_{t-H+i}\|$. By the formula for $E_{t-H+i}$ and the fact that $\|e_{t-H+j}\| \leq \Eg$ for all $j$, it is clear that $$\| E_{t-H+i} \| \leq H\Eg $$
for all $i$. Furthermore, we have assumed that we are on the high-probability event that $\|g_{t}\| = \O(\s(\sqrt{d} + \sqrt{\log\frac{T}{\g}}))$ for all $1\leq t\leq T$ (see Lemma \ref{thm: |g| whp}). This implies that $$ \| G_{t-H+i} \| = \O(H\s(\sqrt{d} + \sqrt{\log\frac{T}{\g}})) = \tilde{\O}(\s)$$ for all $i$. Thus we can choose $\bf{G}$ such that $ \| G_{t-H+i} \| \leq \mathbf{G} \leq cH\s(\sqrt{d} + \sqrt{\log\frac{T}{\g}})$ for some universal constant $c$.

Finally, we claim that $\| S_{t-H+i} \| \leq \frac{(1+L\n)^i - 1}{L\n} \cdot (\mathbf{G} + H\Eg)$ for all $i$. Since $S_{t-H} = 0$, the claim holds for $i=0$. We induct:
\begin{align}
    \| S_{t-H+i+1} \| &\leq \| S_{t-H+i} \| + \| \nb^2 \L^*(\zeta_{t-H+i}) \| (\|G_{t-H+i}\| + \|E_{t-H+i}\| + \n \| S_{t-H+i} \|) \nn \\[5pt]
    &\leq (1 + L\n) \|S_{t-H+i} \| + L (\mathbf{G} + H\Eg) \label{eq: |S| 1} \\[5pt]
    &\leq \frac{(1 + L\n)^{i+1} - (1 + L\n)}{L\n} \cdot (\mathbf{G} + H\Eg)  + (\mathbf{G} + H\Eg) \\[5pt]
    &= \frac{(1 + L\n)^{i+1} - 1}{L\n} \cdot (\mathbf{G} + H\Eg).
\end{align}
This completes the induction. Since $i \leq H$, we have $\|S_{t-H+i}\| \leq \frac{(1+L\n)^H - 1}{L\n}(\mathbf{G} + H\Eg)$. Since $\n = o(1)$, by expanding, we see that $(1+L\n)^H = 1 + \O(HL\n)$. It follows that
$$ \| S_{t-H+i} \| \leq \frac{1 + \O(HL\n) - 1}{L\n} (\mathbf{G} + H\Eg) = \O(H(\mathbf{G} + H\Eg)) =  \O\left(H^2 (\s (\sqrt{d} + \sqrt{\log\frac{T}{\g}}) + \Eg)\right).$$
To summarize, we have:
\begin{equation} \label{eq: EGS summary}
    \| E_{t-H+i} \| = \O(H\Eg) \hspace{.5in} \| G_{t-H+i} \| = \O(H(\sqrt{d} + \sqrt{\log\frac{T}{\g}})\s) \hspace{.25in} \| S_{t-H+i} \| = \O\left(H^2 (\sqrt{d} + \sqrt{\log\frac{T}{\g}}) \s + H^2 \Eg\right).
\end{equation}

\paragraph{Claim 2:} We have the decomposition
$$ \mu_{t-H+i} = \bar{\mu}_{t-H+i} - \n G^\mu_{t-H+i} - \n E^\mu_{t-H+i} + \n^2 S^\mu_{t-H+i}, $$
where $\bar{\mu}_{t-H+i}$ is independent of all of the Gaussian perturbations $g_{t-H+j}$ and the remaining terms in the expression are given by recursive definitions below.

\paragraph{Proof of Claim 2:} The claim holds trivially when $i=-1$, so we induct on $i$. In what follows, to avoid notational clutter, we replace the index $t-H+i$ by $i$. Define $\bar{\psi}_{i} = (\bar{\th}_{i}, \: \bar{\mu}_{i-1})$, and set $\p_1 m_i = \p_1 m(\bpsi_i)$ and $\p_2 m_i = \p_2 m(\bpsi_i)$. Lastly, define
$$ U_i = \begin{bmatrix} G_i + E_i - \n S_i \\ G^\mu_{i-1} + E^\mu_{i-1} - \n S^\mu_{i-1} \end{bmatrix}. $$
We have:
\begin{align}
    \mu_{i+1} &= m(\th_{i+1}, \: \mu_{i}) \nn \\[5pt]
    &= m(\bar{\th}_{i+1} - \n G_{i+1} - \n E_{i+1} + \n^2 S_{i+1}, \: \: \bar{\mu}_{i} - \n G^\mu_{i} - \n E^\mu_{i} + \n^2 S^\mu_{i}) \nn \\[5pt]
    &= m(\bar{\psi}_{i+1}) \nn \\
    &\phantom{= } - \n \partial_1 m(\bar{\psi}_{i+1}) (G_{i+1} + E_{i+1}) - \n \partial_2 m(\bar{\psi}_{i+1}) (G^\mu_{i} + E^\mu_{i}) \nn \\
    &\phantom{= } + \n^2 \p_1 m(\bar{\psi}_{i+1}) S_{i+1} + \n^2 \p_2 m(\bar{\psi}_{i+1}) S^\mu_i + \n^2 \nb^2 m(\omega_{i+1})[\begin{bmatrix} G_{i+1} + E_{i+1} - \n S_{i+1} \\ G^\mu_{i} + E^\mu_{i} - \n S^\mu_{i} \end{bmatrix}, \begin{bmatrix} G_{i+1} + E_{i+1} - \n S_{i+1} \\ G^\mu_{i} + E^\mu_{i} - \n S^\mu_{i} \end{bmatrix}] \label{eq: taylor expand m} \\[5pt]
    &= m(\bar{\psi}_{i+1}) - \n ((\p_1 m_{i+1}) G_{i+1} + (\p_2 m_{i+1}) G^\mu_i) - \n ((\p_1 m_{i+1}) E_{i+1} + (\p_2 m_{i+1}) E^\mu_{i+1}) \nn \\
    &\phantom{= } + \n^2 ( \nb^2 m(\omega_{i+1})[U_{i+1}, U_{i+1}] + (\p_1 m_{i+1}) S_{i+1} + (\p_2 m_{i+1}) S^\mu_i)
\end{align}
Thus if we set 
\begin{align*}
    \bar{\mu}_{i+1} &= m(\bpsi_{i+1}) \\[5pt]
    G^\mu_{i+1} &= (\p_1 m_{i+1}) G_{i+1} + (\p_2 m_{i+1}) G^\mu_i \\[5pt]
    E^\mu_{i+1} &= (\p_1 m_{i+1}) E_{i+1} + (\p_2 m_{i+1}) E^\mu_i \\[5pt]
    S^\mu_{i+1} &= \nb^2 m(\omega_{i+1})[U_{i+1}, U_{i+1}] + (\p_1 m_{i+1}) S_{i+1} + (\p_2 m_{i+1}) S^\mu_i,
\end{align*}
then $\bar{mu}_{i+1}$ has the desired independence property and we have recusrive definitions for $G^\mu_{i+1}$, $E^\mu_{i+1}$, and $S^\mu_{i+1}$. This completes the proof of Claim 2.

Before moving on, we remark that it can be easily seen via induction that
\begin{align} 
    G^\mu_i &= \sum_{j=1}^i (\p_2 m_i) \cdots (\p_2 m_{j+1}) (\p_1 m_j) G_j \nn \\[5pt]
    &= \sum_{j=1}^i (\p_2 m_i) \cdots (\p_2 m_{j+1}) (\p_1 m_j) \sum_{l=0}^{j-1} g_l \nn \\[5pt]
    &= \sum_{l=0}^{i-1} \underbrace{\left(\sum_{j=l+1}^i (\p_2 m_i) \cdots (\p_2 m_{j+1}) (\p_1 m_j)\right)}_{M_{il}} g_l. \label{eq: Gmu}
\end{align}
We also have that
\begin{equation}
    \| M_{il} \| \leq \sum_{j=l+1}^i \| \partial_2 m_i \| \cdots \| \partial_2 m_{j+1} \| \| \partial_1 m_j \| 
    \leq \sum_{j=l+1}^i \d^{i - j} B 
    \leq B \sum_{j=0}^\infty \d^j = \frac{B}{1-\d}. \label{eq: Mil bound}
\end{equation}
This uniform constant bound on $\|M_{il}\|$ will be useful later.

We will now bound $\| G^\mu_i \|$, $\|E^\mu_i\|$, and $\|S^\mu_i\|$. When $i=0$, all of these quantities are 0, which accounts for all of our base cases. We proceed inductively for each one. We first claim $\| G^\mu_i \| \leq B \mathbf{G} \cdot \frac{1 - \d^i}{1-\d}$ (recall that $\mathbf{G}$ was chosen so that $\|G_i\| \leq \mathbf{G} \leq cH\s(\sqrt{d} + \sqrt{\log\frac{T}{\g}})$:
\begin{align*}
    \| G^\mu_{i+1} \| &\leq \| \p_1 m_{i+1} \| \| G_{i+1} \| + \| \p_2 m_{i+1} \| \|G^\mu_i \| \\[5pt]
    &\leq B \mathbf{G} + \d \cdot B \mathbf{G} \frac{1-\d^i}{1-\d} \\[5pt]
    &= B\mathbf{G} \frac{1-\d^{i+1}}{1-\d}.
\end{align*}
This completes the induction. Since $\d < 1$, we have $\|G^\mu_i\| \leq \frac{B\mathbf{G}}{1-\d} = \tilde{\O}(\s)$ for all $i$. By the exact same logic and the fact that $\|E_i\| \leq H\Eg$ for all $i$, we also find that $\|E^\mu_i\| = \O(\frac{BH\Eg}{1-\d})$ for all $i$.

Lastly, define $\mathbf{S} = c'H^2((\sqrt{d} + \sqrt{\log\frac{T}{\g}})\s + \Eg)$ for some universal constant $c'$ so that $\|S_i\| \leq \mathbf{S}$ for all $i$; this can be done by \eqref{eq: EGS summary}. We claim that $\| S^\mu_i \| \leq B' \mathbf{S} \frac{1-\d^i}{1-\d}$ for all $i$. We have our base case $i=0$, so we induct. Observe that, by definition of $U_{i+1}$, we have
\begin{align}
    \| U_{i+1} \|^2 &\leq ( \|G_{i+1} \| + \|E_{i+1}\| + \n \|S_{i+1}\| )^2 + ( \|G^\mu_{i+1} \| + \|E^\mu_{i+1}\| + \n \|S^\mu_{i}\| )^2 \nn \\[5pt]
    &\leq c(\mathbf{G}^2 + \Eg^2 + \n^2 \mathbf{S}^2 + \n^2 \| S^\mu_{i} \|^2) \label{eq: |Smu| 1} \\[5pt]
    &= o(\mathbf{S}) \label{eq: |Smu| 2}
\end{align}
Here \eqref{eq: |Smu| 1} holds by the bounds on $\|G_i\|$, $\|E_i\|$, etc., and the elementary inequality $(x+y+z)^2 = \O(x^2 + y^2 + z^2)$ for any $x,y,z\geq 0$. Inequality \eqref{eq: |Smu| 2} holds since $\mathbf{G}, \Eg, \|S^\mu_i\| = \O(\mathbf{S})$, and $\mathbf{S} = o(1)$.

Now, by the recursive definition of $S^\mu_{i+1}$ and the bounds on $\nb^2 m$, we have:
\begin{align}
    \| S^\mu_{i+1} \| &\leq C \| U_{i+1} \|^2 + B\mathbf{S} + \d \|S^\mu_i\| \nn \\[5pt]
    &= o(\mathbf{S}) + B\mathbf{S} + \d \cdot B'\mathbf{S} \frac{1 - \d^i}{1-\d} \label{eq: |Smu| 3} \\[5pt]
    &\leq B'\mathbf{S} + \d \cdot B'\mathbf{S} \frac{1-\d^i}{1-\d} \nn \\[5pt]
    &= B'\mathbf{S} \frac{1-\d^{i+1}}{1-\d}. \nn
\end{align}
$B'$ can be selected properly in \eqref{eq: |Smu| 3} since the $o(\mathbf{S})$ term is vanishingly small compared to $\mathbf{S}$. This completes the induction. Summarizing our results, we have
\begin{equation} \label{eq: EGSmu summary}
    \| E^\mu_i \| = \O(\frac{BH\Eg}{1-\d}), \hspace{.25in} \| G^\mu_i \| = \O(H\s(\sqrt{d} + \sqrt{\log\frac{T}{\g}})), \hspace{.25in} \| S^\mu_i \| = \tilde{\O}(\s) + \O(\Eg).
\end{equation}

With this in mind, we can can now write:
\begin{align}
    \psi_{t-H+i} - \psi_{t} &= \begin{bmatrix} \bar{\th}_{t-H+i} - \n G_{t-H+i} - \n E_{t-H +i} + \n^2 S_{t-H+i} - (\bar{\th}_{t} - \n G_t - \n E_t + \n^2 S_t) \\ \bar{\mu}_{t-H+i-1} - \n G^\mu_{t-H+i-1} - \n E^\mu_{t-H+i-1} + \n^2 S^\mu_{t-H+i-1} - (\bar{\mu}_{t-1} - \n G^\mu_{t-1} - \n E^\mu_{t-1} + \n^2 S^\mu_{t-1}) \end{bmatrix} \nn \\[5pt]
    &= \underbrace{\begin{bmatrix} \bar{\th}_{t-H+i} - \bar{\th}_t \\ \bar{\mu}_{t-H+i-1} - \bar{\mu}_{t-1} \end{bmatrix}}_{\Delta \underline{\psi_i}}
    + \n \underbrace{\begin{bmatrix} G_t - G_{t-H+i} \\ G^\mu_{t-1} - G^\mu_{t-H+i-1} \end{bmatrix}}_{\Delta G_i}
    + \n \underbrace{\begin{bmatrix} E_t - E_{t-H+i} \\ E^\mu_{t-1} - E^\mu_{t-H+i-1} \end{bmatrix}}_{\Delta E_i}
    + \n^2 \underbrace{\begin{bmatrix} S_t - S_{t-H+i} \\ S^\mu_{t-1} - S^\mu_{t-H+i-1} \end{bmatrix}}_{\Delta S_i}. \nn
\end{align}
Recall that $v_{i:j}$ denotes the matrix whose columns are $v_i$ through $v_j$ for $i \geq j$: $$v_{i:j} = \begin{bmatrix} v_i & v_{i-1} & \cdots & v_j \end{bmatrix}.$$ Using this notation, define
\begin{equation*}
    \underline{\Delta \psi} = \underline{\Delta \psi}_{H-1: 0}, \hspace{.25in} \Delta G = \Delta G_{H-1:0}, \hspace{.25in} \Delta E = \Delta E_{H-1:0}, \hspace{.25in} \Delta S = \Delta S_{H-1:0}.
\end{equation*}
Recalling the definition of $\overline{\Delta \psi}$, we then have
\begin{equation*}
    \overline{\Delta \psi} = \underline{\Delta \psi} + \n \Delta G + \n \Delta E + \n^2 \Delta S.
\end{equation*}
Using the same Weyl's inequality calculation as in \eqref{eq: m jac err 3}, we the have
\begin{equation} \label{eq: weyl for dpsi}
    \| \overline{\Delta \psi}^\dagger \| \leq \frac{ \| (\underline{\Delta \psi} + \n \Delta G)^\dagger \| }{1 - \| (\underline{\Delta \psi} + \n \Delta G)^\dagger \| (\n \|\Delta E\| + \n^2 \|\Delta S\|)}.
\end{equation}
By \eqref{eq: EGS summary}, \eqref{eq: EGSmu summary}, and the definitions of $\Delta E$ and $\Delta S$, we have
\begin{equation}
    \| \Delta E \| \leq \| \Delta E \|_F = \O(\frac{BH^{3/2}}{1-\d}\Eg) \hspace{.5in} \| \Delta S \| \leq \| \Delta S \|_F = \tilde{\O}(\s) + \O(\Eg).
\end{equation}
Thus we have reduced our problem to showing that $\underline{\Delta \psi} + \n \Delta G$ has full rank and bounding $\| (\underline{\Delta \psi} + \n \Delta G)^\dagger \|$ with high probability. Note that since $\underline{\Delta \psi} + \n \Delta G \in \R^{2d \times H}$ and $H \geq 2d$, it suffices to show that for any vector $v\in \R^{2d}$ with $\| v \| = 1$, we have
\begin{equation} \label{eq: pseudoinv goal}
    \|(\underline{\Delta \psi} + \n \Delta G)^\T v\| \geq \n\s.
\end{equation}

First, observe that $\n \Delta G^\T v$ is Gaussian. Furthermore, for any mean 0 Gaussian vector $g$ and deterministic vector $w$, by Lemma \ref{thm: |x + g|} we have
\begin{equation} \label{eq: goal reduction 1}
    \P(\| w + g \| \geq s) \geq \P(\|g\| \geq s).
\end{equation}
Thus it suffices to lower bound $\| \Delta G^\T v\|$ with high probability. By the homogeneity of the Gaussian, it suffices to prove the result for $\s = 1$.

Again to avoid notational clutter, we will changes indices $t-H+i \mapsto i$. Let $v = (v_1, v_2)$ with $v_1, \, v_2 \in \R^d$. Observe that $((\Delta G)^\T v)_i = (G_t^\T v_1 + G^\mu_{t-1}{}^\T v_2) + G_i^\T v_1 + G^\mu_{i-1}{}^\T v_2$. Define
$$ \widetilde{\Delta G^\T v} = \begin{bmatrix} | \\ G_i^\T v_1 + G^\mu_{i-1}{}^\T v_2 \\ | \end{bmatrix}_{i=0}^{H-1} \in \R^H.$$
We will begin by analyzing $\widetilde{\Delta G^\T v}$ and address the constant offset to this later. First, by definition of $G_i$ and $G^\mu_i$, we have:
\begin{equation} \label{eq: ith entry}
    G_i^\T v_1 + G^\mu_{i-1}{}^\T v_2 = \sum_{l=0}^{i-3} g_l^\T (v_1 + M_{il}^\T v_2) + g_{i-2}^\T (v_1 + (\partial_1 m_{i-1})^\T v_2) + g_{i-1}^\T v_1.
\end{equation}
We now consider two cases. Recall that we have assume that $\smin(\partial_i m) \geq \alpha$ for all $i$ and some constant $\alpha > 0$.

\paragraph{Case 1: $\|v_1\|\geq \frac\alpha4$} In this case, we write
\begin{equation} \label{eq: case 1 decomp 1}
G_i^\T v_1 + G^\mu_{i-1}{}^\T v_2 = \sum_{l=0}^{i-2} g_l^\T w_{il} + g_{i-1}^\T v_1
\end{equation}
with $w_{il} = v_1 + M_{il}^\T v_2$. Define $u = \frac{v_1}{\|v_1\|}$, and for each $i,l$ form the unique decomposition $w_{il} = a_{il} u + w^\perp_{il}$ with $u^\T w^\perp_{il} = 0$. Define $\gt_l = g_l^\T u$, and note that since $\|u\|=1$ and $g_l \iid \N(0, I_d)$, we have $\gt \sim \N(0, I_H)$ where $\gt$ is the vector with entries $\gt_l$. With this notation, \eqref{eq: case 1 decomp 1} becomes
\begin{align}
    G_i^\T v_1 + G^\mu_{i-1}{}^\T v_2 = \sum_{l=0}^{i-2} a_{il} \gt_l + \underbrace{\sum_{l=0}^{i-2} g_l^\T w^\perp_{il}}_{Z_i} + \| v_1 \| \gt_{i-1} \label{eq: case 1 decomp 2}
\end{align}
Note that $Z_i$ is Gaussian and independent of $\gt$, since
\begin{align*}
    \E Z_i \gt_j = \sum_{l=0}^{i-2} \E[(w^\perp_{il}{}^\T g_l)(g_j^\T u)] = \sum_{l=0}^{i-2} w^\perp_{il}{}^\T \E g_l g_j^\T u = \sum_{l=0}^{i-2} \I_{\{l = j\}} w^\perp_{il}{}^\T u = 0
\end{align*}
by definition of $w^\perp_{il}$. Now from \eqref{eq: case 1 decomp 2}, we can write
$$ \widetilde{\Delta G^\T v} = \tilde{A} \gt + Z, $$
where $Z \indep \gt$ and $\tilde{A}$ is lower triangular with 0s along the diagonal, $\|v_1\|$ on the first subdiagonal, and the other entries given by $a_{il}$. That is,
$$ \tilde{A}_{il} = \begin{cases} 0 & l \geq i \\ \|v_1\| & l = i-1 \\ a_{il} & l \leq i-2\end{cases}.
$$
Furthermore, by applying \eqref{eq: case 1 decomp 2} to $i = t$, we have $$G_t^\T v_1 + G^\mu_{t-1}{}^\T v_2 = \sum_{l=0}^{t-1} b_l \gt_l + \underbrace{\sum_{l=0}^{t-2} g_l^\T w^\perp_{tl}}_{z} $$
for some coefficients $b_l$ and $z$ Gaussian and independent of $\gt$. Let $b \in \R^H$ denote the vector with entries $b_l$, so that 
$$ G_t^\T v_1 + G^\mu_{t-1}{}^\T v_2 = b^\T \gt + z. $$
It then follows that
\begin{align*}
    (\Delta G)^\T v &= (G_t^\T v_1 + G^\mu_{t-1}{}^\T v_2) \I - \widetilde{\Delta G^\T v} \\
    &= \I(b^\T \gt + z) - \tilde{A}\gt - Z \\
    &= (\I b^\T - \tilde{A})\gt + Z',
\end{align*}
where $Z' = z\I - Z$ is a mean-0 Gaussian vector independent of $\gt$.

We now claim that there exists a constant $H_1 = \O(1)$ such that for all $H \geq H_1$, $\s_{\sqrt{H}}(\I b^\T - \tilde{A}) \geq \frac{\alpha}{10}$. By Theorem~2.1 of \citep{Zhu2019-Rank1}, for any $k$, we have $\s_k(\I b^\T - \tilde{A}) \geq \s_{k+1}(\tilde{A})$, so it suffices to show that $\s_{\sqrt{H}+1}(\tilde{A}) \geq \frac{\alpha}{10}$. By \citep{Horn2012-Matrix}, Corollary~7.3.6, if $M$ is any matrix and $M'$ is a matrix obtained by deleting a row or column of $M$, for any $k$ we have $\s_k(M) \geq \s_k(M')$. Thus if we define $A$ to be the submatrix of $\tilde{A}$ obtained by deleting the first row and last column of $\tilde{A}$, it suffices to show that $\s_{\sqrt{H}+1}(A) \geq \frac{\alpha}{10}$.

Observe that $A$ is lower triangular with diagonal entries $\|v_1\| \geq \frac\a4$. Furthermore, all of the entries of $A$ are bounded by a constant: $|A_{ii}| = \|v_1\| \leq 1$, and
$$|a_{il}| \leq \|w_{il}\| \leq \|v_1\| + \|M_{il}\|\|v_2\| \leq 1 + \frac{B}{1-\d} \leq \frac{2B}{1-\d} \equiv c,$$
where the penultimate inequality holds by \eqref{eq: Mil bound} and the final inequality holds by assuming WLOG that $B \geq 1$. It follows that $\|A\|_F \leq c(H-1)$. Furthermore, since $A$ is lower triangular, we have
$$ \prod_{i=1}^{H-1} \s_i(A) = | \mathrm{det}(A) | = \prod_{i=1}^{H-1} \|v_1\| \geq \left(\frac{\alpha}{4}\right)^{H-1}. $$
Thus we can apply Lemma \ref{thm: sv ineq} with $n=H-1$, $\beta = \frac\alpha4$, $c = \frac{2B}{1-\d}$, and $k=\sqrt{H} + 1$. This implies that there exists a constant $H_1 = \O((\log\frac{B}{\alpha(1-\d)})^2)$
such that for all $H \geq H_1$, we have $\s_{\sqrt{H}+1}(A) \geq \frac\alpha8 \geq \frac{\alpha}{10}$ as desired.

We will now show that $(\Delta G)^\T v$ has a similar decomposition in the other case (when $v_1$ is small), and after doing so, we can proceed with a unified analysis.

\paragraph{Case 2: $\|v_1\| < \frac\alpha4$} Observe that since we chose $\|v\|=1$, we must have $$\|v_2\| = \sqrt{1 - \|v_1\|^2} \geq \sqrt{1 - (\alpha/4)^2} \geq \sqrt{15}/4$$ since $\alpha \leq 1$. Furthermore, note that
$$ \|v_1 + (\partial_1 m_{i-1})^\T v_2 \| \geq \smin(\partial_1 m_{i-1}) \|v_2\| - \|v_1\| \geq \alpha \frac{\sqrt{15}}{4} - \frac\alpha4 = \frac{\sqrt{15} - 1}{4}\alpha. $$
Define $\tilde{v}_{i-2} = v_1 + (\partial_1 m_{i-1})^\T v_2$; the above calculation shows that $\| \tilde{v}_{i-2} \| \geq \frac\alpha2$. We now proceed similarly to Case~1. As before, we have
\begin{equation} \label{eq: case 2 decomp 1}
    G_i^\T v_1 + G^\mu_{i-1}{}^\T v_2 = \sum_{l=0}^{i-3} g_l^\T w_{il} + g_{i-2}^\T \tilde{v}_{i-2} + g_{i-1}^\T v_1
\end{equation}
with $w_{il} = v_1 + M_{il}^\T v_2$. For each $i, l$, define $u_i = \tilde{v}_i / \| \tilde{v}_i \|$, and write
\begin{align*}
    w_{il} &= a_{il} u_l + w^\perp_{il},  \hspace{.25in} u_l^\T w^\perp_{il} = 0 \\
    v_1 &= a_i u_i + v^\perp_{1, i}, \hspace{.25in} u_i^\T v^\perp_{1,i} = 0.
\end{align*}
Set $\gt_i = g_i^\T u_i$ and let $\gt$ be the vector with entries $\gt_i$. Since $\|u_i\|=1$ and $g_i \iid \N(0, I_d)$, we have $\gt \sim \N(0, I_H)$. We can now rewrite \eqref{eq: case 2 decomp 1} as
\begin{equation} \label{eq: case 2 decomp 2}
    G_i^\T v_1 + G^\mu_{i-1}{}^\T v_2 = \sum_{l=0}^{i-3} a_{il} \gt_l + \| \tilde{v}_{i-2} \| \gt_{i-2} + a_{i-1} \gt_{i-1} + \underbrace{\sum_{l=0}^{i-3} g_l^\T w^\perp_{il} + g_{i-1}^\T v^\perp_{1, i-1}}_{Z_i}.
\end{equation}
Again, $Z_i$ is Gaussian and independent of $\gt$:
\begin{align}
    \E Z_i \gt_j &= \E[(\sum_{l=0}^{i-3} g_l^\T w^\perp_{il} + g_{i-1}^\T v^\perp_{1, i-1})(g_j^\T u_j)] \nn \\[5pt]
    &= \sum_{l=0}^{i-3} w^\perp_{il}{}^\T \E[g_l g_j^\T ] u_j + v^\perp_{1,i-1} \E[g_{i-1}g_j^\T] u_j \nn \\[5pt]
    &= \sum_{l=0}^{i-3} \I_{l=j} w^\perp_{il}{}^\T u_j + \I_{i-1=j} v^\perp_{1,i-1}{}^\T u_j \label{eq: termwise 0}\\[5pt]
    &= 0 \nn
\end{align}
where each term in \eqref{eq: termwise 0} is 0 by definition of $w^\perp_{il}$ and $v^\perp_{1, i-1}$. Furthermore, specializing \eqref{eq: case 2 decomp 2} to $i=t$, we find that there is a vector of constant coefficients $b$ such that
$$ G_t^\T v_1 + G^\mu_{t-1}{}^\T v_2 = b^\T \gt + z $$
where $z$ is a mean-0 Gaussian independent of $\gt$. Defining $\widetilde{(\Delta G)^\T v}$ as in Case 1, we then have:
$$ (\Delta G)^\T v = (\I b^\T - \tilde{A})\gt + Z' $$
where $Z' = z\I - Z$ is a mean-0 Gaussian vector independent of $\gt$, and the lower-triangular matrix $\tilde{A}$ is defined by
$$ \tilde{A}_{il} = \begin{cases} 0 & l \geq i \\ a_i & l = i-1 \\ \|\tilde{v}_{i-2} \| & l = i-2 \\ a_{il} & l \leq i-3 \end{cases}. $$

We now claim that there is a constant $H_2$ such that for all $H \geq H_2$, $\s_{\sqrt{H}}(\I b^\T - \tilde{A}) \geq \frac{\alpha}{10}$. Define $\hat{A}$ to be $\tilde{A}$ with its first row and lost column deleted. Let $D = \textrm{diag}(a_i)_{i=1}^{H-1}$ be the diagonal al $\hat{A}$, and finally define $A$ to be the matrix obtained by deleting the first row and last column of $\hat{A} - D$. Note that $A \in \R^{(H-2)\times (H-2)}$ is lower triangular with diagonal entries $\|\tilde{v}_{i-2}\|$, $i=2,\ldots,H-1$. Furthermore, we have:
\begin{align}
    \s_{\sqrt{H}}(\I b^\T - \tilde{A}) &\geq \s_{\sqrt{H}+1}(\tilde{A}) \label{eq: zhu} \\
    &\geq \s_{\sqrt{H}+1}(\hat{A}) \label{eq: horn 1} \\
    &\geq \s_{\sqrt{H}+1}(\hat{A} - D) - \max_i |a_i| \label{eq: weyl} \\
    &\geq \s_{\sqrt{H}+1}(A) - \max_i |a_i| \label{eq: horn 2} \\
    &\geq \s_{\sqrt{H}+1}(A) - \frac\a4 \label{eq: a_i bound}.
\end{align}
Here, \eqref{eq: zhu} follows from \citep{Zhu2019-Rank1} Theorem~2.1; \eqref{eq: horn 1} and \eqref{eq: horn 2} follow \citep{Horn2012-Matrix} Corollary~7.3.6; \eqref{eq: weyl} holds by Weyl's inequality for singular values and the fact that $\s_1(D) = \max_i |a_i|$; and \eqref{eq: a_i bound} holds because $|a_i| \leq \|v_1\| \leq \frac\alpha4$ by definition of $a_i$. It therefore suffices to show that $\s_{\sqrt{H}+1}(A) - \frac\alpha4 \geq \frac{\alpha}{10}$.

Note that $A$ has entries which are bounded by a constant: $|A_{ii}| = \|\tilde{v}_{i-2}\| \leq \|v_1\| + \|\partial_1 m_{i-1}\|\|v_2\| \leq 1 + B$, and $$ |a_{il}| \leq \|w_{il}\| \leq 1 + \frac{B}{1-\d} \leq \frac{2B}{1-\d} \equiv c $$
as in the previous case, so $\|A\|_F \leq c(H-1)$. Furthermore, since $A$ is lower triangular, its determinant is the product of its diagonal entries, and therefore
$$ \prod_{i=1}^{H-2} \s_i(A) = |\mathrm{det}(A)| = \prod_{i=2}^{H-1} \| \tilde{v}_{i-2} \| \geq \left(\frac{\sqrt{15}-1}{4}\alpha\right)^{H-2}. $$
Thus we can apply Lemma \ref{thm: sv ineq} with $n = H-1$, $\beta = \frac{\sqrt{15}-1}{4}\alpha$, and $c = \frac{2B}{1-\d}$ to conclude that $\s_{\sqrt{H}+1}(A) - \frac\alpha4 \geq \frac{\sqrt{15}-1}{8}\alpha - \frac\alpha4 \geq \frac{\alpha}{10}$ for all $H\geq H_2$ with $H_2 = \O((\log\frac{B}{\alpha(1-\d)})^2)$.

In both cases, we have that $(\Delta G)^\T v = M\gt + Z'$, where $M \in \R^{H\times H}$ is a matrix with $\s_{\sqrt{H}}(M) \geq \frac{\alpha}{10}$, $\gt \sim \N(0, I_H)$ and $Z'$ is a mean-0 Gaussian with $Z' \indep \gt$. Thus by Lemma \ref{thm: |x + g|} and Lemma \ref{thm: bernstein |g| lb} with $k = \sqrt{H}$, for all $H \geq \max\{H_1, H_2\}$ we have:
\begin{align*}
    \P( \| (\Delta G)^\T v \| \leq \frac{\alpha}{10} H^{1/4} - r ) &= \P( \| M\gt + Z' \| \leq \alpha H^{1/4} - r) \\
    &\leq \P( \| M\gt \| \leq \alpha H^{1/4} - r) \\
    &\leq 2 \exp(-c' \frac{r^2}{(\alpha/10)^2}).
\end{align*}
Then for any $\e > 0$, setting $r = \frac{\alpha}{10\sqrt{c'}} \sqrt{\log \frac{T}{\g} + 2d \log\frac{3}{\e}}$, we have that
\begin{equation} \label{eq: fixed v lb} \| (\Delta G)^\T v \| \geq \frac{\alpha}{10}H^{1/4} - \frac{\alpha}{10\sqrt{c'}} \sqrt{\log\frac{T}{\g} + 2d \log\frac{3}{\e}} \hspace{.25in} \textrm{with probability } \geq 1-\frac{2\g}{T} \left(\frac{\e}{3}\right)^{2d} \end{equation}
for any fixed $v$.

Now let $\{v^{\textrm{net}}_i\}_{i=1}^{(3/\e)^{2d}} \subseteq \R^{2d}$ be an $\e$-net for $S^{2d-1}$. (An $\e$-net with $(3/\e)^{2d}$ elements exists by, e.g., \citep{Vershynin2018-HDP} Corollary 4.2.13.) By taking a union bound of \eqref{eq: fixed v lb} over each $v^{\textrm{net}}_i$ in the net, we have that \eqref{eq: fixed v lb} holds simultaneously for all $v^{\textrm{net}}_i$ with probability at least $1-\frac{2\g}{T}$. A further union bound shows that this holds over the entire $T$ steps of the trajectory with probability at least $1-\frac{\g}{T}$.

Next, consider any $v\in S^{2d-1}$ and choose $v^{\textrm{net}}_i$ such that $\|v - v^{\textrm{net}}_i\| \leq \e$. We have
\begin{equation} \label{eq: e net 1}
    \| (\Delta G)^\T v \| \geq \| (\Delta G)^\T v_i \| - \| (\Delta G)^\T (v_i - v) \| \geq \frac{\alpha}{10}H^{1/4} - \frac{\alpha}{10\sqrt{c'}} \sqrt{\log\frac{1}{\g} + 2d \log\frac{3}{\e}} - \| \Delta G \|_F \e.
\end{equation}
We can bound $\| \Delta G \|_F$:
\begin{align*}
    \| \Delta G \|_F^2 &= \sum_{i=0}^{H-1} (\| G_H - G_i \|^2 + \| G^\mu_{H-1} - G^\mu_{i-1} \|^2) \\
    &\leq 2 \sum_{i=0}^{H-1} (\|G_H\|^2 + \|G_i\|^2 + \|G^\mu_{H-1}\|^2 + \|G^\mu_{i-1}\|^2).
\end{align*}
We showed previously that $\|G_i\|, \|G^\mu_i\| = \O(\frac{B}{1-\d}(\sqrt{d} + \sqrt{\log\frac{T}{\g}})H)$ for all $i$ with probability at least $1-\g$ over the whole trajectory. Thus $$\| \Delta G \|_F^2 = \O(\frac{B}{1-\d}\sqrt{\log\frac{T}{\g}}H^{3/2})$$ with probability at least $1-\g$. By another union bound, combining this inequality with \eqref{eq: e net 1} implies that with probability at least $1-3\g$ we have
\begin{equation} \label{eq: e net 2}
    \| (\Delta G)^\T v \| \geq \frac{\alpha}{10}H^{1/4} - \frac{\alpha}{10\sqrt{c'}} \sqrt{\log\frac{1}{\g'} + 2d \log\frac{3}{\e}} - \frac{cB}{1-\d}\sqrt{\log\frac{T}{\g}}H^{3/2} \e.
\end{equation}
Setting $\e = H^{-11/8}$, \eqref{eq: e net 2} becomes
\begin{align}
\| (\Delta G)^\T v \| &\geq \frac{\alpha}{10}H^{1/4} - \frac{\alpha}{10\sqrt{c'}} \sqrt{\log\frac{1}{\g'} + 2d (\log 3 + \frac{11}{8}\log H)} - \frac{cB}{1-\d}\sqrt{\log\frac{T}{\g}}H^{1/8} \\[5pt]
&\geq \frac{\alpha}{10}H^{1/4} - \frac{c_1 B}{1-\d} \sqrt{d\log\frac{T}{\g}} H^{1/8} \label{eq: e net 3}
\end{align}
for some universal constant $c_1$. The inequality \eqref{eq: e net 3} $\geq 1$ is quadratic in $H^{1/8}$, and with a simple application of the quadratic formula we see that \eqref{eq: e net 3} $\geq 1$ whenever $H \geq H_3$ for some $H_3 = \O(\frac{B^8 d^4}{\alpha^8(1-\d)^8}(\log\frac{T}{\g})^4)$.

Combining \eqref{eq: e net 3} with \eqref{eq: goal reduction 1}, we have shown that for $H \geq \max\{2d, H_1, H_2, H_3\} = \O(\frac{B^8}{\alpha^8(1-\d)^8}(\log\frac{T}{\g})^4)$, with probability at least $1-3\g$, equation \eqref{eq: pseudoinv goal} holds at every step in the optimization trajectory. (Recall that it was sufficient to prove this for the special case $\s=1$ by homogeneity.)

We are almost done. Plugging this upper bound into \eqref{eq: weyl for dpsi}, we have
\begin{align}
    \| \overline{\Delta \psi}^\dagger \| &\leq \frac{ \| (\underline{\Delta \psi} + \n \Delta G)^\dagger \| }{1 - \| (\underline{\Delta \psi} + \n \Delta G)^\dagger \| (\n \|\Delta E\| + \n^2 \|\Delta S\|)} \nn \\[5pt]
    &\leq \frac{ \frac{1}{\n\s} }{ 1 - \frac{1}{\n\s}( c\n\frac{BH^{3/2}}{1-\d}\Eg + \n^2 (\tilde{\O}(\s) + \O(\Eg)) )} \nn \\[5pt]
    &\leq \frac{\frac{1}{\n\s}}{ 1 - c'\frac{BH^{3/2}}{1-\d} \frac{\Eg}{\s} } \label{eq: |dpsidag| 1} \\[5pt]
    &= \O(\frac{1}{\n\s}) \label{eq: |dpsidag| 2},
\end{align}
where \eqref{eq: |dpsidag| 1} holds because $\n(\tilde{\O}(\s) + \O(\Eg)) = o(\Eg)$ and \eqref{eq: |dpsidag| 2} holds provided that $\s \geq 2c'\frac{BH^{3/2}\Eg}{1-\d}$. Given our eventual choices of $\n$ and $\s$ and the resulting bound on $\Eg$ from Lemma \ref{thm: lt grad error}, both of these conditions will indeed hold. Thus we have that $\overline{\Delta \psi}$ is full rank and $\| \overline{\Delta \psi}^\dagger \| = \O(\frac{1}{\n\s})$, as desired.
\end{proof}

\ltjacerr*

\begin{proof}
Throughout this proof, $\mk$ denotes $\mk(\th_t, \mu_{t-1})$ and $\mu^*$ denotes $\mu^*(\th_t)$. We also seek to evaluate $\dmk$ at the point $(\th_t, \mu_{t-1})$ as well as $\dmus$ at $\th_t$. To avoid notational clutter, we will drop the dependence on the time $t$, so $\th$ denotes $\th_t$ and $\mu$ denotes $\mu_{t-1}$.

As $\wh{\dmus} = \lim_{k\rightarrow\infty} \wh{\dmk}$ and $\| \dmk - \dmus \| \rightarrow 0$ by Lemma \ref{thm: dmk vs dmus}, we have
\begin{align*}
    \| \wh{\dmus} - \dmus \| &= \lim_{k\rightarrow \infty} \| \wh{\dmk} - \dmus \| \\[5pt]
    &\leq \lim_{k \rightarrow \infty} \| \wh{\dmk} - \dmk \| + \| \dmk - \dmus \| \\[5pt]
    &= \lim_{k \rightarrow \infty} \| \wh{\dmk} - \dmk \|.
\end{align*}
Thus it suffices to bound $\| \wh{\dmk} - \dmk \|$ and take the limit as $k\rightarrow \infty$.

Define $B' = \frac{B}{1-\d}$ and $c_1 = C(1 + B')D$. Take $c_2$ to be a constant such that $\| \mu_t - \mu^*(\th_t) \| \leq c_2 BG (\log \frac{1}{\n})^2 \n$, which exists by Lemma \ref{thm: mu_t vs mu*_t}. Finally, define $\b = c_2 C(1+B') BG(\log \frac{1}{\n})^2 \n + 2B' \Em$, and set $\d' = \d + \Em$.

We claim that $\|\dmk - \hdmk\| \leq c_1 \sum_{i=0}^{k-1} \d^i (\d')^{k-1-i} + \b \frac{1 - (\d')^k}{1 - \d'}$. For $k=0$, both $\dmk$ and $\hdmk$ are $0$, so the claim holds trivially. We induct:
\begin{align}
    \| \dmkp - \hdmkp \| &\leq \| \p_1 m(\th, \mk) - \wh{\p_1 m(\th, \mu)} \| + \| \p_2 m(\th, \mk) \dmk - \wh{\p_2 m(\th, \mu)} \hdmk \| \nn \\[10pt]
    &\leq \| \p_1 m(\th, \mk) - \p_1 m(\th, \mu) \| + \| \p_1 m(\th, \mu) - \wh{\p_1 m(\th, \mu)} \| \nn \\
    &+ \| \p_2 m(\th, \mk) - \p_2 m(\th, \mu)\| \| \dmk \| + \| \p_2 m(\th ,\mu) - \wh{\p_2 m(\th, \mu)} \| \| \dmk \| \nn \\
    &+ \| \wh{\p_2 m(\th, \mu)} \| \| \dmk - \hdmk \| \nn \\[5pt]
    &\leq C \| \mk - \mu \| + \Em + C \| \mk - \mu \| B' + \Em B' + (\d + \Em) \| \dmk - \hdmk \| \label{eq: dmk 1} \\[5pt]
    &= C(1 + B') \| \mk - \mu \| + (1 + B') \Em + (\d + \Em) \| \dmk - \hdmk \| \nn \\[5pt]
    &\leq C(1 + B') (\| \mk - \mu^* \| + \| \mu^* - \mu \|) + (1 + B') \Em + (\d + \Em) \| \dmk - \hdmk \| \nn \\[5pt]
    &\leq C(1 + B') (D \d^k + c_2 BG(\log\frac{1}{\n})^2 \n) + 2B' \Em + \d' \| \dmk - \hdmk \| \label{eq: dmk 2} \\[5pt]
    %
    %
    &\leq c_1 \d^k + \b + \d' (c_1 \sum_{i=0}^{k-1} \d^i (\d')^{k-1-i} + \b \frac{1-(\d')^k}{1-\d'}) \label{eq: dmk 3} \\[5pt]
    &= c_1 ( \d^k + \sum_{i=0}^{k-1} \d^i (\d')^{k-i} ) + \b (1 + \frac{\d' - (\d')^{k+1}}{1-\d'}) \nn \\[5pt]
    &= c_1 \sum_{i=0}^k \d^i (\d')^{k-i} + \b \frac{1 - (\d')^{k+1}}{1-\d'}. \nn
\end{align}
In the above, \eqref{eq: dmk 1} holds by Assumption \ref{assumption: 2nd derivs}; \eqref{eq: dmk 2} holds by Lemma \ref{thm: m^k to mu^*}; and \ref{eq: dmk 3} holds by definition of $c_1$, $c_2$, $\b$, $\d'$, and the inductive hypothesis. Since $\Em \geq 0$, we have that $\d' \geq \d$. Furthermore, since $\Em = o(1)$, we may assume that $0 < \d' < 1$. It follows that
\begin{align*}
    \|\dmk - \hdmk\| &\leq c_1 \sum_{i=0}^{k-1} \d^i (\d')^{k-1-i} + \b \frac{1 - (\d')^k}{1 - \d'} \\[5pt]
    &\leq c_1 k (\d')^{k-1} + (c_2 C(1+B')BG(\log \frac{1}{\n})^2 \n + (1+B') \Em) \frac{1}{1-\d'} \\[5pt]
    &= \O\left( k(\d')^k + \frac{C B^2 G}{(1-\d)^2}(\log \frac{1}{\n})^2 \n + \frac{B}{(1-\d)^2}\Em \right).
\end{align*}
Taking $k\rightarrow\infty$, we find that
\begin{equation*}
    \| \wh{\dmus} - \dmus \| = \O\left( \frac{C B^2 G}{(1-\d)^2}(\log \frac{1}{\n})^2 \n + \frac{B}{(1-\d)^2}\Em \right) \equiv \Emu
\end{equation*}
as desired. Substituting the expression for $\Em$, we see that the $(\log\frac{1}{\n})^2\n$ term does not contribute to leading order, so we have
\begin{equation} \label{eq: Emu def}
    \Emu = \O\left(\frac{C\sqrt{H} B^3 G^2}{(1-\d)^2} (\log\frac{1}{\n})^4 \frac{\n}{\s} + \frac{B^3 G H}{(1-\d)^2} (\log\frac{1}{\n})^2 \frac{\be}{\n \s^2} \right).
\end{equation}
\end{proof}

\ltgraderr*

\begin{proof}
We have
\begin{align*}
    \nb \L^*(\th) &= \overbrace{\int \nb_\th \ell(z, \th) p(z, \mu^*(\th)) \, dz}^{\nb_1 \L^*(\th)} + \overbrace{\int \ell(z, \th) \dmus^\T \nb_\mu p(z, \mu^*(\th)) \, dz}^{\nb_2 \L^*(\th)} \\[5pt]
    \hnb \L^*(\th) &= \underbrace{\int \nb_\th \ell(z, \th) p(z, \hmu) \, dz}_{\hnb_1 \L^*(\th)} + \underbrace{\int \ell(z, \th) \wh{\dmus}^\T \nb_\mu p(z, \hmu) \, dz}_{\hnb_2 \L^*(\th)}
\end{align*}
We write $\| \nb \L^* - \hnb \L^* \| \leq \| \nb_1 \L^* - \hnb_1 \L^* \| + \| \nb_2 \L^* - \hnb_2 \L^* \|$ and bound each of these terms separately. For the remainder of this proof, we will assume that $\| \hmu - \mu^* \| = o(1)$. By the result of Lemma \ref{thm: mu_t vs mu*_t}, this will be the case when $\n = o(1)$.

Let $\s_0^2 = \| \Sigma \|$, and let $B^d(R)$ denote the Euclidean ball in $\R^d$ of radius $R$. For the first term, we have
\begin{align}
    \| \nb_1 \L^* - \hnb_1 \L^* \| &\leq \lmax \int |p(z, \mu^*) - p(z, \hmu)| \, dz \nonumber \\[5pt]
    &\leq \lmax \left[ \int_{\| z - \hmu \| \leq R} | p(z, \mu^*) - p(z, \hmu) | \, dz + \int_{\| z - \hmu \| > R}  p(z, \hmu) \, dz + \int_{\| z - \hmu \| > R} p(z, \mu^*) \, dz \right] \nonumber \\[5pt]
    &\leq \lmax \left[ L_{p, 2} \| \mu^* - \hmu \| \mathrm{vol}(B^d(R)) + \int_{\| z - \hmu \| > R}  p(z, \hmu) \, dz + \int_{\| z - \mu^* \| > R - \| \hmu - \mu^* \|} p(z, \mu^*) \, dz \right] \label{eq: grad err 1} \\[5pt]
    &= \lmax [ L_{p, 2} \| \mu^* - \hmu \| \mathrm{vol}(B^d(R)) \nonumber \\
    &+ \P_{z \sim \N(\hmu, \Sigma)}( \| z - \hmu \| \geq R ) + \P_{z \sim \N(\mu^*, \Sigma)}( \| z - \mu^* \| \geq R - \| \hmu - \mu^* \|) ] \nonumber \\[5pt]
    &\leq \lmax[ 6L_{p, 2} R^d \| \mu^* - \hmu \|  \nonumber \\
    &+ c_1\exp\{ -c_2 (R - \s_0 \sqrt{d})^2 / \s_0^2 \}] + c_1\exp\{ -c_2 (R - \s_0 \sqrt{d} - \| \hmu - \mu^* \|)^2 / \s_0^2 \} \label{eq: grad err 2} \\[5pt]
    &= \O( R^d \| \mu^* - \hmu \| + \exp\{ -c_2 (R - \s_0 \sqrt{d} - \| \hmu - \mu^* \|)^2 / \s_0^2 \} ) \label{eq: grad err 3}
\end{align}
for any $R \geq \s_0 \sqrt{d} + \| \hmu - \mu^* \|$. In the above, \eqref{eq: grad err 1} holds because
$$\| z - \hmu \| > R \: \Longrightarrow \: \| z - \mu^* \| > \| z - \hmu \| - \| \hmu - \mu^* \| > R - \| \hmu - \mu^* \|.$$
Equation \eqref{eq: grad err 2} holds by the inequality $\mathrm{vol}(B^d(R)) \leq 6 R^d$, and by Lemma \ref{thm: |g| tail}. If we then set
$$R = \s_0\sqrt{d} + \| \hmu - \mu^* \| + \frac{\s_0}{\sqrt{c_2}} \sqrt{ \log \frac{1}{\| \hmu - \mu^* \|} },$$
substituting into \eqref{eq: grad err 3} yields
\begin{align}
    \| \nb_1 \L^* - \hnb_1 \L^* \| &= \O( 3^d [(\s_0\sqrt{d})^d + \| \hmu - \mu^* \|^d + (\frac{\s_0}{\sqrt{c_2}} (\log \frac{1}{\| \hmu - \mu^* \|} )^{1/2})^d ]\| \hmu - \mu^* \| + \| \hmu - \mu^* \|) \label{eq: grad err 4} \\[5pt]
    &= \O\left( \left(\log \frac{1}{\| \hmu - \mu^* \|}\right)^{d/2} \| \hmu - \mu^* \| \right). \nonumber
\end{align}
Equation \eqref{eq: grad err 4} holds by the elementary inequality $(a+b+c)^d \leq 3^d (a^d + b^d + c^d)$ for any $a, b, c, d \geq 0$.

The bound on the second gradient term is similar to the first. First, for the Gaussian density $p$, note that $\nb_\mu p(z, \mu) = \Sigma^{-1}(\mu - z) p(z, \mu)$. Using this fact, we have
\begin{align}
    \| \nb_2 \L^* - \hnb_2 \L^* \| &\leq \lmax \int \| \dmus^\T \nb_\mu p(z, \mu^*) - \wh{\dmus}^\T \nb_\mu p(z, \hmu) \| \, dz \nonumber \\[5pt]
    &\leq \lmax [ \| \dmus^\T \| \int \| \nb_\mu p(z, \mu^*) - \nb_\mu p(z, \hmu) \| \, dz + \| \dmus^\T - \wh{\dmus}^\T \| \int \| \nb_\mu p(z, \hmu) \| \, dz ] \nonumber \\[5pt]
    &\leq \lmax[ \frac{B}{1-\d} L_{\nb_\mu p, 2} \| \mu^* - \hmu \| \mathrm{vol}(B^d(R)) + \int_{\|z - \hmu\| > R} \| \nb_\mu p(z, \hmu) \| \, dz + \int_{\|z - \mu^*\| > R - \| \mu^* - \hmu \|} \| \nb_\mu p(z, \mu^*) \| \, dz \nonumber \\
    &+  \| \dmus - \wh{\dmus} \| \int \| \nb_\mu p(z, \hmu) \| \, dz] \nonumber \\[5pt]
    &= \O\left( R^d \| \mu^* - \hmu \| + \int_{\|z - \mu^*\| > R - \| \mu^* - \hmu \|} \| \nb_\mu p(z, \mu^*) \| \, dz + \| \dmus - \wh{\dmus} \| \|\Sigma^{-1/2}\|\sqrt{d}\right). \label{eq: grad err 4.5}
\end{align}
Equation \eqref{eq: grad err 4.5} follows by applying Lemma \ref{thm: int grad p} to $\int \|\nb_\mu p\|$. We bound the integral in the last line separately. For any $r > \s_0\sqrt{d}$, we have
\begin{align}
    \int_{\|z - \mu^* \| > r} \| \nb_\mu p(z, \mu^*) \| \, dz &= \int_{\|z - \mu^*\| > r } \| \Sigma^{-1}(\mu^* - z) \| p(z, \mu^*) \, dz \nonumber \\[5pt]
    &= \E_{z\sim \N(\mu^*, \Sigma)} [ \|\Sigma^{-1} (z - \mu^*)\| \cdot \I\{\| z - \mu^* \| > r\} ] \nonumber \\[5pt]
    &\leq \sqrt{ \E_{z \sim \N(\mu^*, \Sigma)}[ \|\Sigma^{-1}(z-\mu^*) \|^2 ] \cdot \E_{z \sim \N(\mu^*, \Sigma)} [ \I\{\| z - \mu^* \| > r\}^2 ] } \label{eq: grad err 5} \\[5pt]
    &\leq \sqrt{ \|\Sigma^{-1/2}\|^2 \E_{z \sim \N(\mu^*, \Sigma)}[ \| \Sigma^{-1/2}(z-\mu^*)\|^2 ] \cdot \P_{z \sim \N(\mu^*, \Sigma)}(\|z - \mu^* \| > r) } \nonumber \\[5pt]
    &= \| \Sigma^{-1/2} \| \sqrt{ \E_{z \sim \N(0, I_d)} [\|z\|^2] \cdot c_1\exp\{ -c_2 (r - \s_0 \sqrt{d})^2 / \s_0^2 \} } \label{eq: grad err 6} \\[5pt]
    &= \| \Sigma^{-1/2} \| \sqrt{d} \cdot \sqrt{c_1} \exp\{ -c_2 (r - \s_0 \sqrt{d})^2 / 2\s_0^2 \}. \label{eq: grad err 7}
\end{align}
Inequality \eqref{eq: grad err 5} holds by the Cauchy-Schwarz inequality. Equation \eqref{eq: grad err 6} holds because $$z \sim \N(\mu^*, \Sigma) \: \Longrightarrow \: \Sigma^{-1}(z-\mu^*)\sim \N(0, I_d),$$ and by Lemma \ref{thm: |g| tail}. Finally, equation \eqref{eq: grad err 7} holds because $\E_{z\sim \N(0, I_d)} \|z\|^2 = d$. We can now plug \eqref{eq: grad err 7} into \eqref{eq: grad err 4.5}:
\begin{equation*}
    \| \nb_2 \L^* - \hnb_2 \L^* \| = \O\left( R^d \|\mu^* - \hmu\| + \exp\{ -c_2 (R - \| \mu^* - \hmu \| - \s_0 \sqrt{d})^2/2\s_0^2 \} + \| \dmus - \wh{\dmus} \| \|\Sigma^{-1/2}\|\sqrt{d} \right).
\end{equation*}
If we take $R = \s_0\sqrt{d} + \| \hmu - \mu^* \| + \frac{\sqrt{2}\s_0}{\sqrt{c_2}}\sqrt{\log \frac{1}{\| \hmu - \mu^* \|}}$, then by the same logic as was used in Equation \eqref{eq: grad err 4}, we obtain
\begin{equation*}
    \| \nb_2 \L^* - \hnb_2 \L^* \| = \O\left( \left(\log \frac{1}{\|\hmu - \mu^*\|}\right)^{d/2} \| \hmu - \mu^* \| + \| \dmus - \wh{\dmus} \| \|\Sigma^{-1/2}\sqrt{d} \right).
\end{equation*}
Thus the bound on $\| \nb_1 \L^* - \hnb_1 \L^* \|$ can be absorbed into the bound on $\| \nb_2 \L^* - \hnb_2 \L^* \|$, and we have
\begin{equation} \label{eq: grad err 8}
    \| \nb \L^* - \hnb \L^* \| = \O\left( \left(\log \frac{1}{\|\hmu - \mu^*\|}\right)^{d/2} \| \hmu - \mu^* \| + \| \dmus - \wh{\dmus} \| \|\Sigma^{-1/2}\|\sqrt{d} \right).
\end{equation}

Now, when we take $\th = \th_t$ (so $\mu^* = \mu^*(\th_t)$ and we are evaluating $\dmus$ at $\th_t$) and $\hmu = \hmu_t$, by Lemma \ref{thm: mu_t vs mu*_t}, for $t \geq \log_\d \n$, we have
$$\| \hmu_t - \mu^*(\th_t) \| \leq \| \hmu_t - \mu_t \| + \| \mu_t - \mu^*(\th_t) \| \leq \O(\bm{\e} + (\log\frac1\n)^2 \n).$$
Substituting this into \eqref{eq: grad err 8} and using the definition of $\Emu$, we have
\begin{align*}
    \| \nb \L^* - \hnb \L^* \| &= \O\left( \left(\log \frac{1}{\bm{\e} + \left(\log \frac1\n\right)^2 \n}\right)^{d/2} \left(\bm{\e} + \left(\log\frac1\n\right)^2 \n\right) + \sqrt{d}\|\Sigma^{-1/2}\| \Emu  \right) \\[5pt]
    &= \tilde{\O}( \bm{\e} + \n ) + \O( \sqrt{d}\|\Sigma^{-1/2}\| \Emu ).
\end{align*}
Since $\bm{\e} = \O(\n)$, we obtain the desired result. From the expression \eqref{eq: Emu def} for $\Emu$, we see that the $\tilde{\O}(\n)$ term does not contribute to leading order, and we obtain
\begin{equation} \label{eq: Eg 2}
    \| \nb \L^* - \hnb \L^*\| = \O\left(\frac{C B^3 G^2 \sqrt{Hd} \|\Sigma^{-1/2}\|}{(1-\d)^2} (\log\frac{1}{\n})^4 \frac{\n}{\s} + \frac{B^3 G H\sqrt{d} \|\Sigma^{-1/2}\|}{(1-\d)^2} (\log\frac{1}{\n})^2 \frac{\be}{\n \s^2} \right) \equiv \Eg.
\end{equation}
Note that this matches the definition of $\Eg$ given in \eqref{eq: Eg 1}.
\end{proof}

\gderr*

\begin{proof}
We require the additional assumptions that $\| \nb h(x) \| \leq G$ and that $|h(x)| \leq h_{\max}$ for all $x$, and that $G, h_{\max} = \O(1)$.

Since $h$ is $L$-smooth, we have $h(y) \leq h(x) + \nb h(x)^\T (y - x) + \frac{L}{2}\|x-y\|^2$ for any $x, y$. Define $e_t = \wh{\nb}h(x_t) - \nb h(x_t)$, so $\wh{\nb h}(x_t) = \nb h(x_t) + e_t$. Taking $x = x_t$ and $y = x_{t+1}$, we have
\begin{align}
    h(x_{t+1}) &\leq h(x_t) + \nb h(x_t)^\T (x_t - \n (\nb h(x_t) + e_t) - x_t) + \frac{L}{2} \|x_t - \n (\nb h(x_t) + e_t) - x_t\|^2 \nonumber \\[5pt]
    &\leq h(x_t) - \n \| \nb h(x_t) \|^2 + \n \| \nb h(x_t) \| \| e_t \| + \n^2 L \| \nb h(x_t) \|^2 + \n^2 L \| e_t \|^2 \label{eq: gd 1} \\[5pt]
    &\leq h(x_t) - \n \| \nb h(x_t) \|^2 + \n G \mathbf{e} + \n^2 L \| \nb h(x_t) \|^2 + \n^2 L \mathbf{e}^2 \nonumber \\[5pt]
    &= h(x_t) + (\n^2 L - \n) \| \nb h(x_t) \|^2 + \n G \mathbf{e} + \n^2 L \mathbf{e}^2 \nonumber.
\end{align}
Here \eqref{eq: gd 1} holds by the Cauchy-Schwarz inequality. Since $\n = o(1)$, we may assume that $\n - \n^2 L > 0$. Rearranging, it follows that
\begin{equation} \label{eq: gd 2}
    \| \nb h(x_t) \|^2 \leq \frac{h(x_t) - h(x_{t+1}) + \n G \mathbf{e} + \n^2 L \mathbf{e}^2}{\n - \n^2 L}.
\end{equation}
We now sum \eqref{eq: gd 2} from $t = 1$ to $T$. This yields
\begin{align}
    T \min_{1\leq t \leq T} \| \nb h(x_t) \|^2 &\leq \sum_{t=1}^T \| \nb h(x_t) \|^2 \nonumber \\[5pt]
    &\leq \sum_{t=1}^T \frac{h(x_t) - h(x_{t+1}) + \n G \mathbf{e} + \n^2 L \mathbf{e}^2}{\n - \n^2 L} \nonumber \\[5pt]
    &= \frac{h(x_1) - h(x_{T+1})}{ \n - \n^2 L} + T \cdot \frac{ \n G \mathbf{e} + \n^2 L \mathbf{e}^2}{\n - \n^2 L} \nonumber \\[5pt]
    &\leq \frac{2 h_{\max}}{\n - \n^2 L} + T \cdot \frac{ \n G \mathbf{e} + \n^2 L \mathbf{e}^2}{\n - \n^2 L} \nonumber \\[5pt]
    &= \O\left( \frac{h_{\max}}{\n} + GT\mathbf{e} \right). \label{eq: gd 3}
\end{align}
Here \eqref{eq: gd 3} holds since $\n = o(1)$ and $\mathbf{e} = o(1)$. Dividing both sides of \eqref{eq: gd 3} by $T$ yields the desired result.
\end{proof}

\main*

\begin{proof}
First, we remark that in order for Lemma \ref{thm: mu_t vs mu*_t} to hold, we need a ``warm-up" phase of length $\log\frac1\n$. We will always take $\n = \Omega(T^{-2/5})$, in which case this warm-up phase has length $\O(\log T^{2/5}) = \O(\log T) = o(T)$. This does not change the asymptotic length of the trajectory, so we will simply ignore it in the following calculations. We will also assume that all of the required high-probability events hold from each of the previous lemmas, making the following statements hold with probability at least $1 - \O(\g)$.

We split into two (very similar) cases. First, we consider when $\be \geq \frac{1}{T}$.

Suppose that for $s\leq t$, we have that $\| \hnb \L^*(\th_s) - \nb \L^*(\th_s) \| \leq \Eg$ with 
\begin{equation} \label{eq: Eg 3}
\Eg = \O\left(\frac{C B^3 G^2 \sqrt{Hd} \|\Sigma^{-1/2}\|}{(1-\d)^2} (\log\frac{1}{\n})^4 \frac{\n}{\s} + \frac{B^3 G H\sqrt{d} \|\Sigma^{-1/2}\|}{(1-\d)^2} (\log\frac{1}{\n})^2 \frac{\be}{\n \s^2} \right).
\end{equation}
(Again, note that if $s \leq \log \frac{1}{\n}$, then we replace $\hnb \L^*(\th_s)$ and $\nb \L^*(\th_s)$ with $0$ and the above bound holds trivially.) Since $\Eg = o(1)$ and $\|\hnb \L^*(\th_s)\| \leq G + \Eg$, we have $\|\hnb \L^*(\th_s)\| \leq cG$ for some absolute constant $c$. Furthermore, if we require that
\begin{equation} \label{eq: pert condition}
    \s \geq 2c'\frac{BH^{3/2}\Eg}{1-\d},
\end{equation}
then Lemma \ref{thm: finite diff error} holds. Then we have the chain Lemma \ref{thm: finite diff error} $\Rightarrow$ Lemma \ref{thm: lt jac error} $\Rightarrow$ Lemma \ref{thm: lt grad error}, and in particular Lemma \ref{thm: lt grad error} holds with the same $\Eg$ as in \eqref{eq: Eg 3}. Inductively, we see that $\| \hnb \L^*(\th_t) - \nb \L^*(\th_t) \| \leq \Eg$ for all $1\leq t \leq T$.


We can now apply Lemma \ref{thm: gd with errors} with
$\mathbf{e} = \Eg + \O(\s \sqrt{\log\frac{T}{\g}})$.
(The second term accounts for the fact that the perturbation $g_t$ must be included in $\mathbf{e}$, and $\|g_t\| = \O(\s \sqrt{\log\frac{T}{\g}})$.) This yields
\blockcomment{
\begin{equation*}
    \min_{1\leq t\leq T} \| \nb \L^*_t \|^2 = \tilde{\O}\left(\frac{1}{T\n} + \frac{\n}{\s} + \frac{\be}{\n\s^2} + \s\right).
\end{equation*}
}
\begin{equation} \label{eq: almost done}
    \min_{1\leq t\leq T} \| \nb \L^*_t \|^2 = \O\left(\frac{\lmax}{T\n} + G(\Eg + \s\sqrt{\log\frac{T}{\g}})\right) = \O\left(\frac{\lmax\be}{\n} + G\sqrt{\log\frac{T}{\g}} \cdot \s\right)
\end{equation}
where the second bound holds by condition \eqref{eq: pert condition} and $\be \geq 1/T$.

Next, let us analyze the condition \eqref{eq: pert condition}; it takes the form $\s \geq c_1 (\log\frac{1}{\n})^4 \frac{\n}{\s} + c_2 (\log\frac{1}{\n})^2 \frac{\be}{\n\s^2}$. Dividing both sides of this inequality by $\s$ and setting $\n = \sqrt{\be/\s}$, \eqref{eq: pert condition} holds if
\begin{align*}
    c_1(\log\frac{1}{\n})^4 \frac{\n}{\s^2} + c_2(\log\frac{1}{\n})^2\frac{\be}{\n\s^3} &\leq (c_1 + c_2)(\log\frac{1}{\n})^4 \frac{\be^{1/2}}{\s^{5/2}} \leq (c_1 + c_2) (\log \frac{1}{\be})^4 \frac{\be^{1/2}}{\s^{5/2}} \leq 1.
\end{align*}
The rightmost inequality holds when $\s \geq (c_1 + c_2)^{2/5} (\log\frac{1}{\be})^{8/5} \be^{1/5}$. Furthermore, the expressions for the coefficients $c_1$ and $c_2$ are given by
$$ c_1 = \frac{2c'BH^{3/2}}{1-\d} \cdot c\frac{C B^3 G^2 \sqrt{Hd} \| \Sigma^{-1/2}\|}{(1-\d)^2} = \frac{c'' B^4 G^2 H^2 \sqrt{d} \| \Sigma^{-1/2} \|}{(1-\d)^3}$$
$$ c_2 = \frac{2c'BH^{3/2}}{1-\d} \cdot c\frac{B^3 G H \sqrt{d} \| \Sigma^{-1/2} \|}{(1-\d)^2} = \frac{c'' B^4 G H^{5/2} \sqrt{d} \| \Sigma^{-1/2} \|}{(1-\d)^3}.$$
It follows that
$$ (c_1 + c_2)^{2/5} = \O\left( \frac{B^{8/5} G^{4/5} H d^{1/5} \|\Sigma^{-1}\|^{1/5}}{(1-\d)^{6/5}}\right), $$
which finally yields that \eqref{eq: pert condition} holds for
\begin{equation} \label{eq: pert cond 2}
    \s = \Omega\left(\frac{B^{8/5} G^{4/5} H d^{1/5} \|\Sigma^{-1}\|^{1/5}}{(1-\d)^{6/5}} (\log\frac{1}{\be})^{8/5}\be^{1/5}\right).
\end{equation}

Finally, we set $\s = c (\log\frac{1}{\be})^{8/5} \be^{1/5}$ and $\n = \sqrt{\be/\s}$. Observe that in this case, $\frac{\be}{\n} = \be^{1/2}\s^{1/2} = \tilde{\O}(\be^{3/5}) = o(\s)$, so the $\be/\n$ term in \eqref{eq: almost done} can be ignored. Recalling the fact that we must choose $H = \Theta(\frac{B^8 d^4}{\alpha^8(1-\d)^8}(\log\frac{T}{\g})^4)$ in order for Lemma \ref{thm: |dpsi^dag|} to hold, we have
\begin{align*}
    \min_{1\leq t\leq T} \| \nb \L^*_t \|^2 &= \O\left( \frac{B^{8/5} G^{9/5} H d^{1/5} \|\Sigma^{-1}\|^{1/5}}{(1-\d)^{6/5}} \sqrt{\log\frac{T}{\g}} (\log \frac{1}{\be})^{8/5} \be^{1/5} \right) \\[5pt]
    &= \O\left( \frac{B^{9.6} G^{1.8} d^{4.2} \|\Sigma^{-1}\|^{0.2}}{\alpha^8(1-\d)^{9.2}} (\log\frac{T}{\g})^{4.5} (\log \frac{1}{\be})^{1.6} \cdot \be^{1/5} \right) \\[5pt]
    &= \tilde{\O}(\be^{1/5}).
\end{align*}
This completes the case when $\be \geq \frac{1}{T}$. Otherwise, we have $\be < \frac1T$. Starting from \eqref{eq: Eg 3}, WLOG we can replace each occurence of $\be$ with $\frac1T$ and all of the bounds will still hold, so in this case we get
\begin{equation*}
    \O\left( \frac{B^{9.6} G^{1.8} d^{4.2} \|\Sigma^{-1}\|^{0.2}}{\alpha^8(1-\d)^{9.2}} (\log\frac{T}{\g})^{4.5} (\log T)^{1.6} \cdot T^{-1/5} \right) = \tilde{\O}(T^{-1/5}).
\end{equation*}
Since we are always in one case or the other, we always have
\begin{equation*}
    \min_{1\leq t\leq T} \| \nb \L^*_t \|^2 = \tilde{\O}(T^{-\frac15} + \be^{\frac15}).
\end{equation*}
The fact that there are interval of nonzero width for $\n$ and $\s$ follows from the fact that we can multiply or divide both of these by constants close to 1 and not change any of the asymptotics (since the constants can be chosen such that \eqref{eq: pert condition} still holds). This completes the proof.
\end{proof}

\section{EXPERIMENT DETAILS} \label{appendix: experiment details}
For all of the experiments, we did a grid search over the relevant parameters for each method, then chose the best results for that method. The parameters we considered were:

RGD $\rightarrow$ learning rate (lr) \\
DFO $\rightarrow$ learning rate (lr), wait, perturbation size (ps) \\
PerfGD $\rightarrow$ learning rate (lr), wait, horizon ($H$) \\
SPGD $\rightarrow$ learning rate (lr), perturbation size (ps), horizon ($H$) \\

The grid ranges for each parameter were as follows:
\begin{itemize}
    \item lr $\in \{10^{-k/2} \: : \: k = 1, \ldots, 6 \}$
    \item wait $\in \{1, 5, 10, 20\}$
    \item DFO ps $\in \{10^{-k/2} \: : \: k=0,\ldots,3\}$
    \item SPGD ps $\in \{0\} \cup \{ 10^{-k/2} \: : \: k = 0,\ldots,3 \}$
    \item PGD $H \in \{d, \, d+1, \, \ldots, 2d, \, \infty\}$
    \item SPGD $H \in \{2d, \, 2d + 1, \, \ldots, \, 3d, \, \infty\}$
\end{itemize}

For each experiment, we specify the dimension $d$. We also require that $\th \in [-R, R]^d$ for some $R$. If any of the optimization methods took $\th$ outside of this constraint set, we simply clamped $\th$ back to the required range.

Rather than having deterministically bounded error on the mean estimates $\hat{\mu}_t$, we take $\hat{\mu}_t = \mu_t + e_t$, where $e_t \iid \N(0, \s_{\textrm{err}}^2 I)$ are Gaussian error terms which would arise from taking $\hat{\mu}_t$ to be the mean of a finite sample.

For \S \ref{sec: linear m experiment}, we set $d = 5$ and $R = 5$. The matrix $A$ was chosen as $-0.8 \times$ a random PSD matrix. The vector $b$ was set to be $2 \times$ the all 1's vector. We used a time horizon of $T=50$, and the noise on estimating the mean was $\s_{\textrm{err}} = 10^{-3}$. We did 5 trials per scenario. 

For \S \ref{sec: nonlinear m experiment}, we set $d=5$, $R=5$, and the mean estimation noise is still $\s_{\textrm{err}} = 10^{-3}$. We set $A = -0.8 I$, $b$ to be $2 \times$ the all ones vector, and $\delta = 0.684$, and conducted 50 trials. In a small fraction of runs, the SPGD gradient estimate would explode, so we also clipped the gradient if its norm exceeded 10 by normalizing it to a unit vector.

For \S \ref{sec: spam}, we set $d = 2$ and $R = 3$. We had $\mu_0 = [2, 1]^\T$ and $\mu_1 = [1,2]^\T$. The proportion of spammers was $0.5$. We set $\alpha = -2$, the regularization strength to be $10^{-1}$, and $\delta = 0.25$. The mean estimation noise was $\s_{\textrm{err}} = 10^{-3}$, and we conduct 50 trials.

\subsection{Bottleneck}
The long-term mean for $\th$ is given by
$$ \mu^*(\th) = \frac{1}{1 + \th^\T \mu_0} \mu_0. $$
The long-term performative loss is then
$$ \L^*(\th) = \frac{-\th^\T \mu_0}{1 + \th^\T \mu_0} + \frac{\lambda}{2}\|\th\|^2. $$
If there is a long-term distribution, then the mean satisifes the fixed point equation $\mu = (1 - \th^\T \mu) \mu_0$. This equation implies that $\mu = c\mu_0$ for some scalar $c$. Substituting and solving the resulting equation, we see that The long-term mean for $\th$ is given by

$$ \mu^*(\th) = \frac{1}{1 + \th^\T \mu_0} \mu_0. $$

To avoid the denominator blowing up, we would like to enforce some constraints on $\th$ and $\mu$. We accomplish this by setting $\Th = \{ \th \in \R^d \: : \: \th \geq 0 \textrm{ and } \| \th \|_\infty \leq 1/\sqrt{d} \}$. If we also choose $\mu_0$ so that $\| \mu_0 \|_2 = 1$ and $\mu_0 \geq 0$, then we claim that $\|\mu_t\| \leq 1$ and $0 \leq \th^\T \mu_t \leq 1$ for all $t$. It trivially holds for $t = 0$. At time $t + 1$, we have:
\begin{align*}
    \| \mu_{t + 1} \|_2 &= |1 - \th^\T \mu_t| \|\mu_0\|_2 \\
    &= |1 - \th^\T \mu_t|.
\end{align*}
Since $\th^\T \mu_t \geq 0$, we trivially have $1 - \th^\T \mu_t \leq 1$. To see that it is also nonnegative:
\begin{align*}
    1 - \th^\T \mu_t &\geq 1 - \|\th\|_2 \|\mu_t\|_2 \\
    &\geq 1 - \sqrt{d} \frac{1}{\sqrt{d}} \cdot 1 = 0.
\end{align*}
It follows that $\|\mu_{t+1}\|_2 \leq 1$. Furthermore, since $\mu_0 \geq 0$ and $1 - \th^\T \mu_t \geq 0$, we have $\mu_{t+1} \geq 0$ and therefore $0 \leq \th^\T \mu_{t+1} \leq 1$, completing the induction.

A simple calculation yields
\begin{align}
    \nb \L^*(\th) &= \lambda \th - \frac{\mu_0}{(1 + \th^\T \mu_0)^2}, \label{eq: long term grad}\\
    \nb^2 \L^*(\th) &= \lambda I + \frac{2}{(1 + \th^\T \mu_0)^3}\mu_0 \mu_0^\T. \label{eq: long term hess}
\end{align}
If we assume that $\|\mu_0\|_2 = 1$, then \eqref{eq: long term hess} implies that $\L^*$ is convex precisely when $2/(1 + \th^\T \mu_0)^3 \geq -\lambda$ for all $\th$. Since $\mu_0, \th \geq 0$, it follows that $\L^*$ is convex for any nonnegative regularization strength $\lambda$. Thus we should expect (approximate) gradient descent to find the minimizer for this problem.

\section{ADDITIONAL EXPERIMENTS}
\subsection{Extended Results for the Linear Experiment}
Here we extend the results of Section \ref{sec: linear m experiment} as the mean takes longer and longer to settle. For both of the following experiments, we kept the number of model deployments at $T = 50$. Figure~\ref{fig: extended linear 0.001} shows the performance of each algorithm at the same noise level as the previous experiment $(\s_{\textrm{err}} = 10^{-3})$. Figure~\ref{fig: extended linear 0} shows the results with no noise on the mean $(\s_{\textrm{err}} = 0)$.

At the same noise level as the previous experiments, SPGD maintains its superior performance when it takes the distribution 128 steps to settle. However, as the number of steps required for the distribution to settle increases, the noise in mean estimation becomes larger than $\p_2 m$ and SPGD can no longer get an accurate estimate of the long-term loss gradient. However, as the error on $\hat{\mu}_t$ decreases below the size of $\p_2 m$, SPGD is able to form a good estimate of the long-term performative gradient even for extremely slowly adapting distributions, obtaining near-optimal performance even when the distribution takes 512 or even 2048 steps to settle to its long term value.

\begin{figure}
    \centering
    \includegraphics[width=0.5\textwidth]{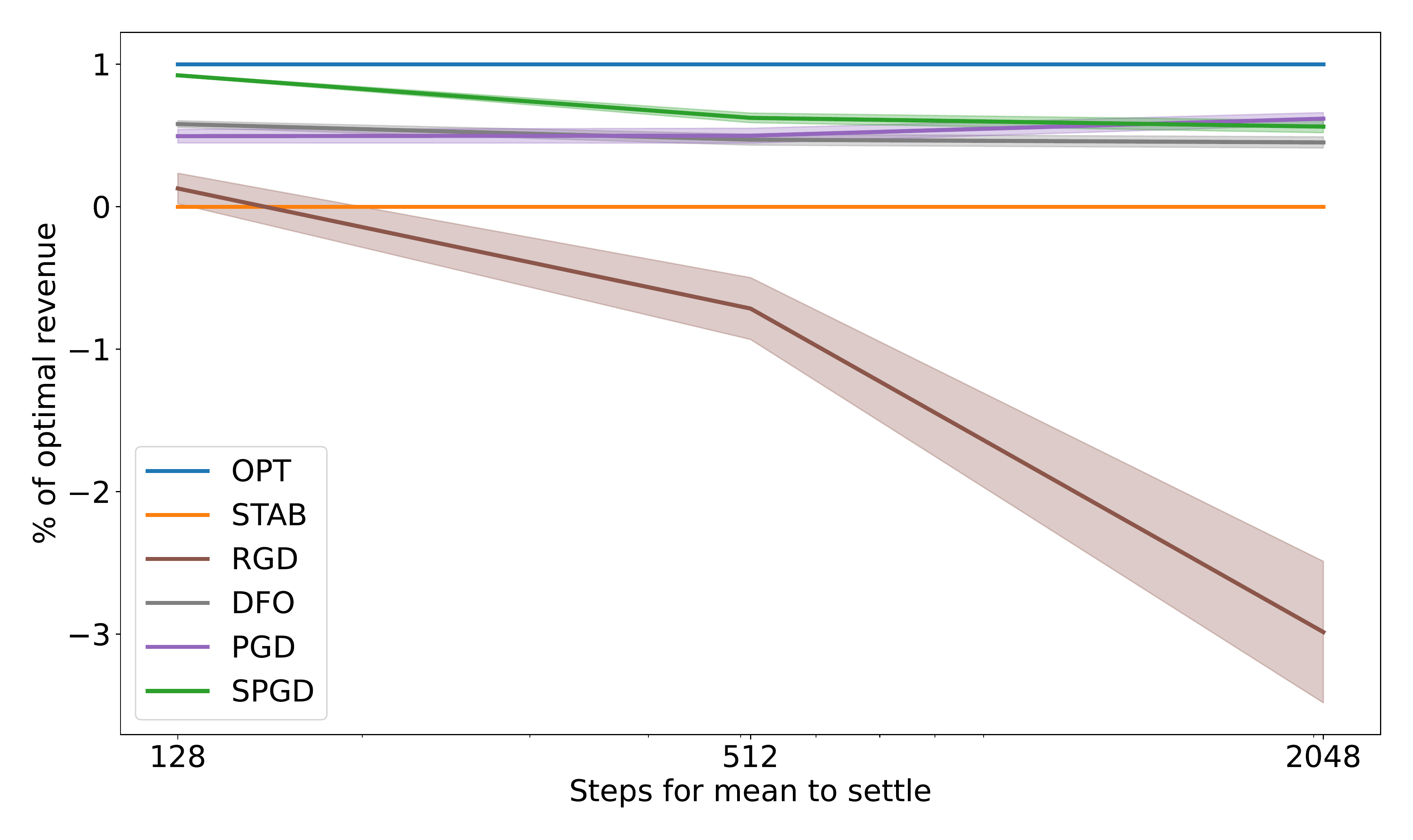}
    \caption{Additional results for the linear $m$ experiment. SPGD retains its superior performance even when the distribution takes 128 steps to settle, but for slow enough dynamics, SPGD eventually degrades to the level of the other algorithms due to the noise in estimating $\mu_t$.}
    \label{fig: extended linear 0.001}
\end{figure}

\begin{figure}
    \centering
    \includegraphics[width=.5\textwidth]{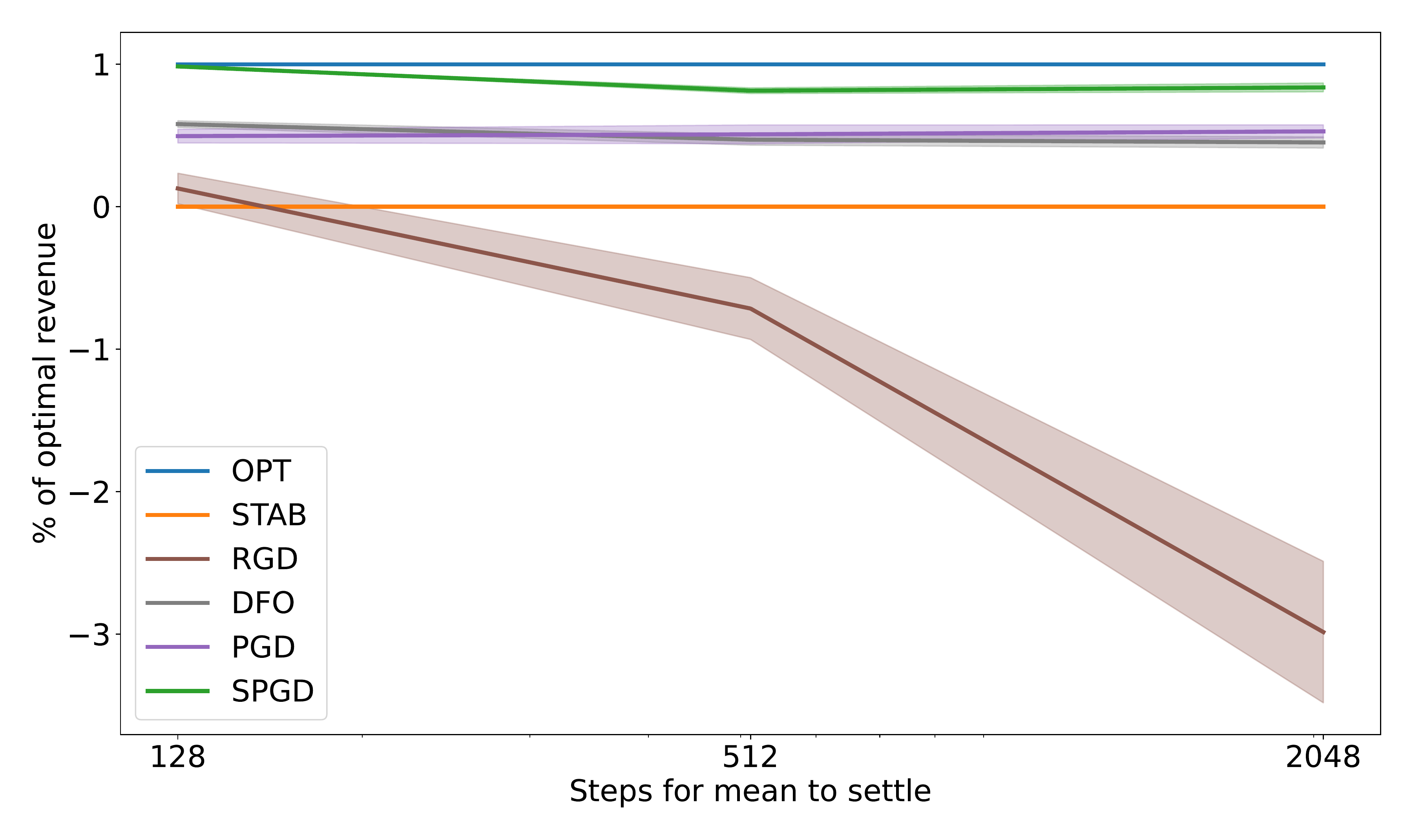}
    \caption{When there is no (or very little) noise in estimating $\mu_t$, SPGD can still get an accurate estimate for the long-term gradient and manages to retain its superior performance even for extremely slowly adapting distributions.}
    \label{fig: extended linear 0}
\end{figure}

\subsection{Dynamics with Oscillations}
We consider another experiment where the distribution parameters do not converge monotonically to their long-term values. We still use the point loss $\ell(z, \th) = -z^\T \th$, but we take $$m(\th, \mu) = \d (A\th + b) - (1-\d) \mu$$ for some fixed $0 < \d < 1$. In this case, $\mk$ oscillates around the long-term value of $\mu^*(\th) = \frac{\d}{2-\d}(A\th + b)$. This oscillation can be seen in Figure~\ref{fig: oscillation demonstration} in dimension $d=2$. Consecutive updates of $\mu$ oscillate back and forth on either side of the long-term value.
\begin{figure}
    \centering
    \includegraphics[width=.5\textwidth]{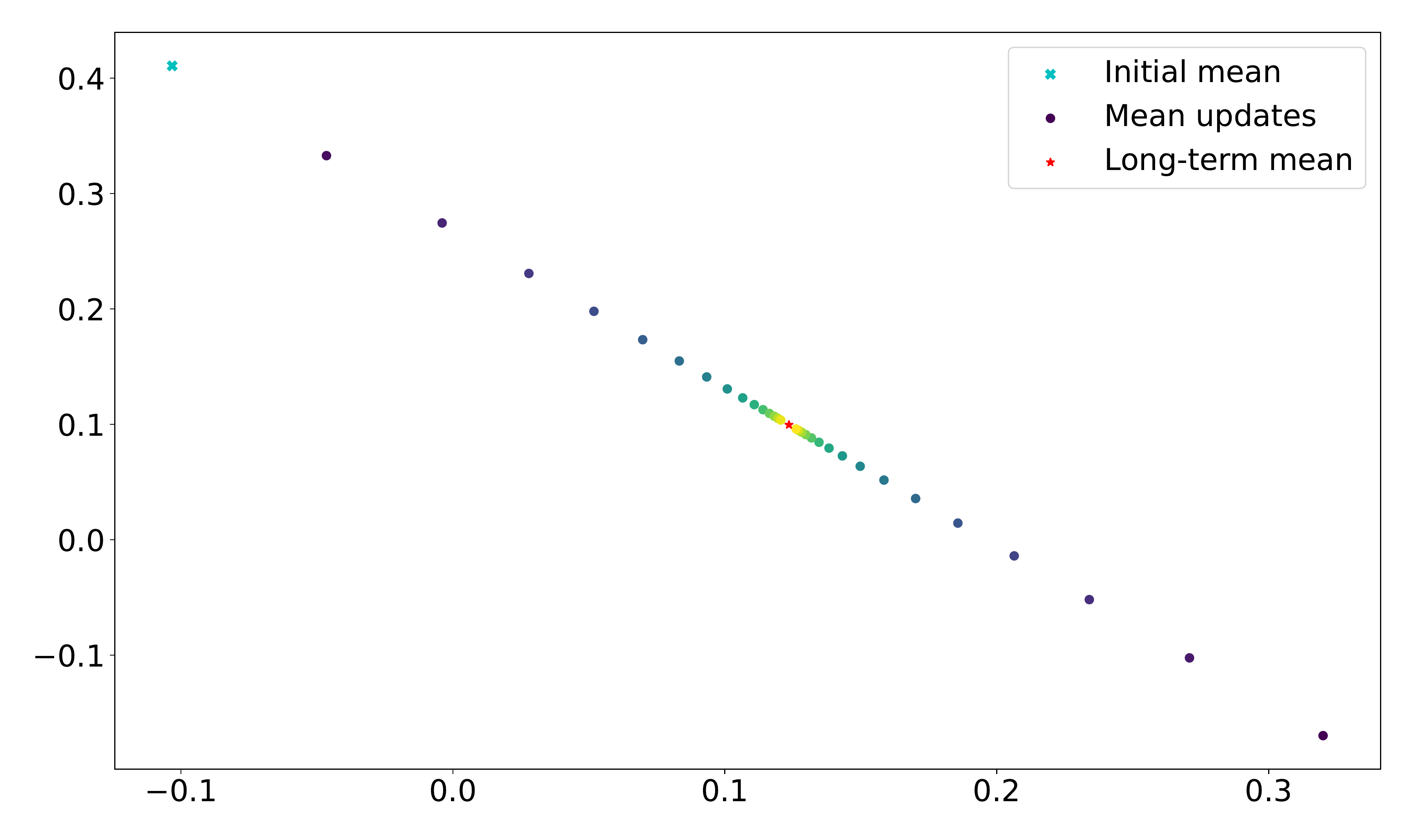}
    \caption{A demonstration of the oscillating distribution dynamics for this experiment.}
    \label{fig: oscillation demonstration}
\end{figure}

In spite of the oscillations present in the dynamics, by choosing a \emph{fixed} base point for the finite difference approximations used to estimate the long-term derivatives, SPGD still performs well in this setting. Figure~\ref{fig: oscillation results} shows the results for $\delta = 0.134$. (Roughly speaking, this corresponds to a situation where it takes 32 steps for the effect of the initial distribution to decay.)
\begin{figure}
    \centering
    \includegraphics[width=.5\textwidth]{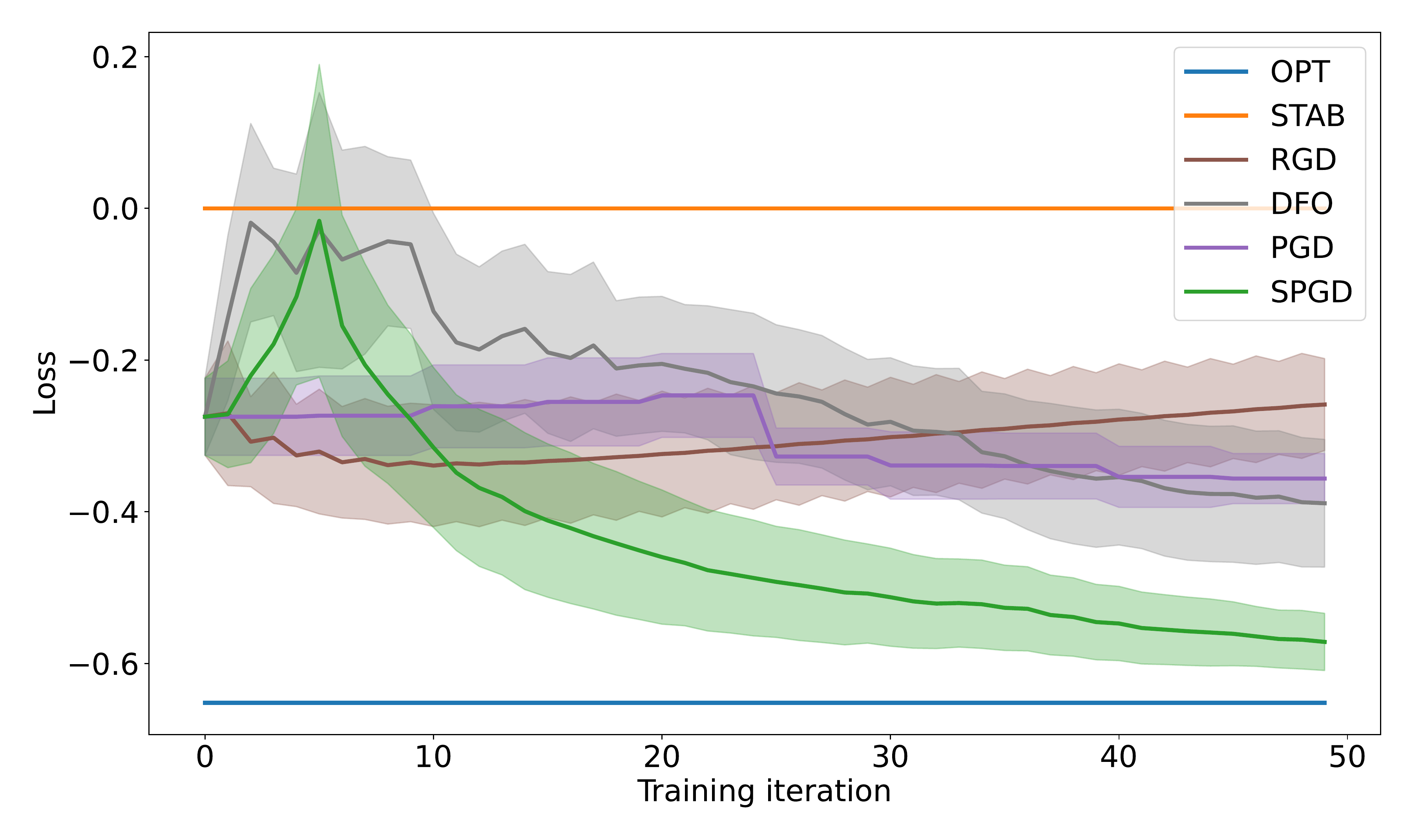}
    \caption{Performance of each of the algorithms for the oscillating distribution dynamics shown in Figure~\ref{fig: oscillation demonstration}. SPGD still outperforms the other algorithms and finds a near-optimal point.}
    \label{fig: oscillation results}
\end{figure}

\end{document}